\documentclass{article} % For LaTeX2e

\usepackage{booktabs} 
\usepackage{graphicx}
\usepackage{url}
\usepackage[utf8]{inputenc}
\usepackage[inline]{enumitem}
\usepackage{xcolor}
\usepackage{stmaryrd}
\usepackage{bibentry}
\usepackage{xspace}
\usepackage{booktabs}
\usepackage{siunitx}
\usepackage{comment}
\usepackage{subcaption}
\usepackage{soul}
\usepackage{comment}
\usepackage{array}
\usepackage{amsthm}
\usepackage{bbm}
\usepackage{placeins}
\usepackage{xcolor}
\usepackage[ruled]{algorithm}
\usepackage[noend]{algorithmic}
\usepackage{makecell}
\usepackage{nicefrac}
\usepackage{enumitem}
\usepackage{graphbox}
\renewcommand{\algorithmicrequire}{\textbf{Input:}}
\renewcommand{\algorithmicensure}{\textbf{Output:}}
\definecolor{ForestGreen}{RGB}{34,139,34}
\definecolor{goldenrod}{rgb}{0.85, 0.65, 0.13}

\newtheorem{assumption}{Assumption}
\newtheorem{proposition}{Proposition}
\newtheorem{definition}{Definition}

\usepackage[pagebackref=true,colorlinks=true,citecolor=blue]{hyperref}
\usepackage[accepted]{icml2022}
\hypersetup{pagebackref=true,colorlinks=true,citecolor=blue,linkcolor=red}
\renewcommand*{\backrefalt}[4]{%
    \ifcase #1 \footnotesize{(Not cited.)}%
    \or        \footnotesize{(p.~#2)}%
    \else      \footnotesize{(pp.~#2)}%
    \fi}

\newtheorem{innercustomgeneric}{\customgenericname}
\providecommand{\customgenericname}{}
\newcommand{\newcustomtheorem}[2]{%
  \newenvironment{#1}[1]
  {%
   \renewcommand\customgenericname{#2}%
   \renewcommand\theinnercustomgeneric{##1}%
   \innercustomgeneric
  }
  {\endinnercustomgeneric}
}

\newcustomtheorem{customthm}{Theorem}
\newcustomtheorem{customlemma}{Lemma}
\newcustomtheorem{customprop}{Proposition}
\usepackage{amsmath,amsfonts,bm,amssymb,amsthm}
\usepackage{mathtools}
\usepackage{optidef}
\usepackage{breqn}
\usepackage{xfrac}

\newcommand*\diff{\mathrm{d}}
\newcommand*\Image{\mathrm{Im}}

\newcommand*\X{\mathcal{X}}

\newcommand*\G{\mathcal{G}}

\newcommand*\D{\mathcal{D}}
\newcommand*\F{\mathcal{F}}
\newcommand*\W{\mathcal{W}}
\newcommand*\Loss{\mathcal{L}}

\newcommand*{\Etr}{\smash{\mathcal{E}_{\scriptsize\textrm{tr}}}}
\newcommand*{\Ead}{\smash{\mathcal{E}_{\scriptsize\textrm{ad}}}}

\newcommand*{\eg}{e.g.\@\xspace}
\newcommand*{\sut}{s.t.\@\xspace}
\newcommand*{\ie}{i.e.\@\xspace}
\newcommand*{\iid}{i.i.d.\@\xspace}
\newcommand*{\wrt}{w.r.t.\@\xspace}
\newcommand*{\aka}{a.k.a.\@\xspace}

\newcommand*{\etc}{etc.\@\xspace}
\newcommand*{\cf}{cf.\@\xspace}

\newtheorem{lemma}{Lemma}

\let\originalleft\left
\let\originalright\right
\renewcommand{\left}{\mathopen{}\mathclose\bgroup\originalleft}
\renewcommand{\right}{\aftergroup\egroup\originalright}

\def\eqref#1{eq.~\ref{#1}}
\def\1{\bm{1}}

\def\vzero{{\bm{0}}}

\def\E{{\mathcal{E}}}

\DeclareMathAlphabet{\mathsfit}{\encodingdefault}{\sfdefault}{m}{sl}
\SetMathAlphabet{\mathsfit}{bold}{\encodingdefault}{\sfdefault}{bx}{n}

\def\gX{{\mathcal{X}}}

\newcommand{\R}{\mathbb{R}}

\newcommand{\Span}{\mathrm{Span}}

\usepackage{cleveref}
\Crefname{assumption}{Assumption}{Assumptions}
\Crefname{lemma}{Lemma}{Lemmas}
\Crefname{definition}{Definition}{Definitions}
\Crefformat{equation}{Eq.~(#2#1#3)}

\makeatletter

\newcommand{\mytag}[2]{%
  \text{#1}%
  \@bsphack
  \begingroup
    \@onelevel@sanitize\@currentlabelname
    \edef\@currentlabelname{%
      \expandafter\strip@period\@currentlabelname\relax.\relax\@@@%
    }%
    \protected@write\@auxout{}{%
      \string\newlabel{#2}{%
        {#1}%
        {\thepage}%
        {\@currentlabelname}%
        {\@currentHref}{}%
      }%
    }%
  \endgroup
  \@esphack
}
\makeatother
\newcommand{\ours}{CoDA\@\xspace}
\newcommand{\ourslone}{\ours-$\smash{\ell_1}$\@\xspace}
\newcommand{\oursltwo}{\ours-$\smash{\ell_2}$\@\xspace}

\begin{document}

\twocolumn[
\icmltitle{Generalizing to New Physical Systems via Context-Informed Dynamics Model} 


\icmlsetsymbol{equal}{*}

\begin{icmlauthorlist}
\icmlauthor{Matthieu Kirchmeyer}{isir,criteo,equal}
\icmlauthor{Yuan Yin}{isir,equal} \\
\icmlauthor{Jérémie Donà}{isir} 
\icmlauthor{Nicolas Baskiotis}{isir}
\icmlauthor{Alain Rakotomamonjy}{criteo,rouen}
\icmlauthor{Patrick Gallinari}{isir,criteo}
\end{icmlauthorlist}

\icmlaffiliation{isir}{CNRS-ISIR, Sorbonne University, Paris, France}
\icmlaffiliation{criteo}{Criteo AI Lab, Paris, France}
\icmlaffiliation{rouen}{Université de Rouen, LITIS, France}
\icmlcorrespondingauthor{Matthieu Kirchmeyer}{matthieu.kirchmeyer@sorbonne-universite.fr}
\icmlcorrespondingauthor{Yuan Yin}{yuan.yin@sorbonne-universite.fr}


\vskip 0.3in
]



\printAffiliationsAndNotice{\icmlEqualContribution} % otherwise use the standard text.

\linepenalty=1000
\everypar{\looseness=-1}

\begin{abstract}
Data-driven approaches to modeling physical systems fail to generalize to unseen systems that share the same general dynamics with the learning domain, but correspond to different physical contexts. 
We propose a new framework for this key problem, context-informed dynamics adaptation (CoDA), which takes into account the distributional shift across systems for fast and efficient adaptation to new dynamics.
CoDA leverages multiple environments, each associated to a different dynamic, and learns to condition the dynamics model on contextual parameters, specific to each environment.
The conditioning is performed via a hypernetwork, learned jointly with a context vector from observed data.
The proposed formulation constrains the search hypothesis space for fast adaptation and better generalization across environments with few samples. 
We theoretically motivate our approach and show state-of-the-art generalization results on a set of nonlinear dynamics, representative of a variety of application domains.
We also show, on these systems, that new system parameters can be inferred from context vectors with minimal supervision.
\end{abstract}

\section{Introduction}
\label{sec:intro}
Neural Network (NN) approaches to modeling dynamical systems have recently raised the interest of several communities leading to an increasing number of contributions. 
This topic was explored in several domains, ranging from simple dynamics \eg Hamiltonian systems \citep{Greydanus2019, Chen2019} to more complex settings \eg fluid dynamics \citep{Kochkov2021,Li2021, Wandel2020}, earth system science and climate \citep{Reichstein2019}, or health \citep{Fresca2020}. 
NN emulators are attractive as they may for example provide fast and low cost approximations to complex numerical simulations \citep{Duraisamy2019, Kochkov2021}, complement existing simulation models when the physical law is partially known \cite{Yin2021Aphynity} or even offer solutions when classical solvers fail \eg with very high number of variables \citep{sirignano2018}.

A model of a real-world dynamical system should account for a wide range of contexts resulting from different external forces, spatio-temporal conditions, boundary conditions, sensors characteristics or system parameters. 
These contexts characterize the dynamics phenomenon.
For instance, in cardiac electrophysiology \citep{Neic2017,Fresca2020}, each patient has its own specificities and represents a particular context. 
In the study of epidemics' diffusion \citep{Shaier2021}, computational models should handle a variety of spatial, temporal or even sociological contexts.
The same holds for most physical problems, \eg forecasting of spatial-location-dependent dynamics in climate \citep{Debezenac2018}, fluid dynamics prediction under distinct external forces \citep{Li2021}, \etc

The physics approach for modeling dynamical systems relies on a strong prior knowledge about the underlying phenomenon. This provides a causal mechanism  which is embedded in a physical dynamics model, usually a system of differential equations, and allows the physical model to handle a whole set of contexts. 
Moreover, it is often possible to adapt the model to new or evolving situations, \eg via data assimilation \citep{kalman1960new,courtier1994strategy}.

In contrast, Expected Risk Minimization (ERM) based machine learning (ML) fails to generalize to unseen dynamics. Indeed, it requires \iid data for training and inference while dynamical observations are non-\iid as the distributions change with initial conditions or physical contexts. 

Thus any ML framework that handles this question should consider other assumptions. 
A common one used \eg in domain generalization \citep{Wang2021DG}, states that data come from several environments \aka domains, each with a different distribution.
Training is performed on a sample of the environments and test corresponds to new ones.
Domain generalization methods attempt to capture problem invariants via a unique model, assuming that there exists a representation space suitable for all the environments. 
This might be appropriate for classification, but not for dynamical systems where the underlying dynamics differs for each environment.
For this problem, we need to learn a function that adapts to each environment, based on a few observations, instead of learning a single domain-invariant function.
This is the objective of meta-learning \citep{Thrun1998}, a general framework for fast adaptation to unknown contexts.
The standard gradient-based methods (\eg \citealp{Finn2017}) are unsuitable for complex dynamics due to their bi-level optimization and are known to overfit when little data is available for adaptation, as in the few-shot learning setting explored in this paper \citep{Mishra2018}. 
Like invariant methods, meta-learning usually handles basic tasks \eg classification; regression on static data or simple sequences and not challenging dynamical systems.

Generalization for modeling real-world dynamical systems is a recent topic. 
Simple simulated dynamics were considered in Reinforcement Learning \citep{Lee2020RL,Clavera2018} while physical dynamics were modeled in recent works \citep{Yin2021LEADS,Wang2021c}.
These approaches consider either simplified settings or additional hypotheses \eg prior knowledge and do not offer general solutions to our adaptation problem (details in \Cref{sec:related_work}).

We propose a new ML framework for generalization in dynamical systems, called \textbf{Co}ntext-Informed \textbf{D}ynamics \textbf{A}daptation (\ours).
Like in domain generalization, we assume availability of several environments, each with its own specificity, yet sharing some physical properties. 
Training is performed on a sample of the environments. 
At test time, we assume access to example data from a new environment, here a trajectory. 
Our goal is to adapt to the new environment distribution with this trajectory. 
More precisely, \ours assumes that the underlying system is described by a parametrized differential equation, either an ODE or a PDE. 
The environments share the parametrized form of the equation but differ by the values of the parameters or initial conditions.
\ours conditions the dynamics model on learned environment characteristics \aka contexts and generalizes to new environments and trajectories with few data.
Our main contributions are the following:
\begin{itemize}
    \item We introduce a multi-environment formulation of the generalization problem for dynamical systems.
    \item We propose a novel context-informed framework, \ours, to this problem. It conditions the dynamics model on context vectors via a hypernetwork.
    \ours introduces a locality and a low-rank constraint, which enable fast and efficient adaptation with few data. 
    \item We analyze theoretically the validity of our low-rank adaptation setting for modeling dynamical systems.
    \item We evaluate two variations of \ours on several ODEs/PDEs representative of a variety of application domains, \eg chemistry, biology, physics. \ours achieves SOTA generalization results on in-domain and one-shot adaptation scenarios. We also illustrate how, with minimal supervision, \ours infers accurately new system parameters from learned contexts.
\end{itemize}
The paper is organized as follows. 
In \Cref{sec:dg_dyn}, we present our multi-environment problem. 
In \Cref{sec:conda_framework}, we introduce the \ours framework.
In \Cref{sec:implementation}, we detail how to implement our framework.
In \Cref{sec:experiments}, we present our experimental results.
In \Cref{sec:related_work}, we present related work.

\section{Generalization for Dynamical Systems}
\label{sec:dg_dyn}
We present our generalization problem for dynamical systems, then introduce our multi-environment formalization. 

\subsection{Problem Setting}
\label{subsec:pb_setting}
We consider dynamical systems that are driven by unknown temporal differential equations of the form:
\begin{equation}
    \frac{\diff x(t)}{\diff t} = f(x(t)), \label{eq:diffeq}%
\end{equation}
where $t \in \mathbb{R}$ is a time index, $x(t)$ is a time-dependent state in a space $\X$ and $f:\X\to T\X$ a function that maps $x(t)\in\X$ to its temporal derivatives in the tangent space $T\X$. 
$f$ belongs to a class of vector fields $\F$. 
$\X\subseteq\mathbb{R}^d$ ($d\in\mathbb{N}^\star$) for ODEs or $\X$ is a space of functions defined over a spatial domain (\eg 2D or 3D Euclidean space) for PDEs.

Functions $f \in \F$ define a space $\smash{\D^f\!(\X)}$ of state trajectories $x: I\to \X$, mapping $t$ in an interval $I$ including $0$, to the state $x(t)\in\X$. 
Trajectories are defined by the initial condition $x(0)\triangleq x_0\sim p(X_0)$ and take the form:%
\begin{equation}
    \forall t\in I, x(t) = x_0 + \int_{0}^{t}f(x(\tau))\diff \tau \in \X \label{eq:inteq}%
\end{equation}
In the following, we assume that $f\in\F$ is parametrized by some unknown attributes \eg physical parameters, external forcing terms which affect the trajectories. 

\subsection{Multi-Environment Learning Problem}
\label{subsec:multi_env_pb}
We propose to learn the class of functions $\F$ with a data-driven \textit{dynamics model} $\smash{g_\theta}$ parametrized by $\smash{\theta\in\mathbb{R}^{d_\theta}}$. 
Given $f\!\in\!\F$, we observe $N$ trajectories in $\smash{\D^f\!(\X)}$ (\cf \Cref{eq:inteq}).

The standard ERM objective considers that all trajectories are \iid 
Here, we propose a multi-environment learning formulation where observed trajectories of $f$ form an environment $e\in\E$.
We denote $f^e$ and $\D^e$ the corresponding function and set of $N$ trajectories.
We assume that we observe training environments $\Etr$, consisting of several trajectories from a set of known functions $\smash{\{f^e\}_{e \in \Etr}}$.

The goal is to learn $\smash{g_\theta}$ that adapts easily and efficiently to new environments $\Ead$, corresponding to unseen functions $\smash{\{f^e\}_{e \in \Ead}}$ (``ad'' stands for adaptation).
We define $\forall e\in\E$ the Mean Squared Error (MSE) loss, over $\smash{\D^e}$ as
\begin{equation}
    \Loss(\theta, \D^e) \triangleq \sum_{i=1}^N  \int_{t\in I}\|f^e(x^{e,i}(t))-g_\theta(x^{e,i}(t))\|^{2}_2\diff t \label{eq:Le}
\end{equation}
In practice, $f^e$ is unavailable and we can only approximate it from discretized trajectories. 
We detail later in \Cref{eq:Le_int} our approximation method based on an integral formulation.
It fits observed trajectories directly in state space.

\section{The \ours Learning Framework}
\label{sec:conda_framework}
We introduce \ours, a new context-informed framework for learning dynamics on multiple environments.
It relies on a general adaptation rule (\Cref{subsec:adapt_rule}) and introduces two key properties: locality, enforced in the objective (\Cref{subsec:locality_optim}) and low-rank adaptation, enforced in the proposed model via hypernetwork-decoding (\Cref{subsec:model}).
The validity of this framework for dynamical systems is analyzed in \Cref{subsec:validity} and its benefits are discussed in \Cref{subsec:discussion}.

\subsection{Adaptation Rule}
\label{subsec:adapt_rule}
The dynamics model $g_\theta$ should adapt to new environments.
Hence, we propose to condition $g_\theta$ on observed trajectories $\D^e, \forall e\in\E$.
Conditioning is performed via an \textit{adaptation network} $\smash{A_\pi}$, parametrized by $\pi$, which adapts the weights of $\smash{g_\theta}$ to an environment $e\in\E$ according to
\begin{equation}
    \smash{\theta^e \triangleq A_\pi(\D^e) \triangleq \theta^c + \delta\theta^e,\quad\pi\triangleq\{\theta^c, \{\delta\theta^e\}_{e\in\E}\}} \label{eq:adaptation_rule}
\end{equation}
$\theta^c\in\mathbb{R}^{d_\theta}$ are shared parameters, used as an initial value for fast adaptation to new environments.
$\delta\theta^e\in\mathbb{R}^{d_\theta}$ are environment-specific parameters conditioned on $\D^e$.

\subsection{Constrained Optimization Problem}
\label{subsec:locality_optim}
Given the adaptation rule in \Cref{eq:adaptation_rule}, we introduce a constrained optimization problem which learns parameters $\pi$ such that $\forall e\in\E, \delta\theta^e$ is small and $g$ fits observed trajectories.
It introduces a locality constraint with a norm $\|\cdot\|$:
\begingroup
\setlength\abovedisplayskip{5pt}
\setlength\belowdisplayskip{5pt}
\begin{equation*}
    \min_{\pi}\sum_{e\in \E} \|\delta\theta^e\|^2 \text{~\sut~} \forall x^e(t) \in \D^e, \frac{\diff x^e(t)}{\diff t} = g_{\theta^c+\delta\theta^e}(x^e(t))
\end{equation*}
\endgroup
We consider an approximation of this problem which relaxes the equality constraint with the MSE loss $\Loss$ in \Cref{eq:Le}.
\begingroup
\setlength\abovedisplayskip{5pt}
\setlength\belowdisplayskip{5pt}
\begin{equation}
    \min_{\pi} \sum_{e\in \E} \Big(\Loss(\theta^c+\delta\theta^e, \D^e) + \lambda \|\delta\theta^e\|^2\Big) \label{eq:arm_relaxed}
\end{equation}
\endgroup
$\lambda$ is a hyperparameter.
For training, we minimize \Cref{eq:arm_relaxed} \wrt $\pi$ over training environments $\Etr$.
After training, $\theta^c$ is freezed.
For adaptation, we minimize \Cref{eq:arm_relaxed} over new environments $\Ead$ \wrt $\smash{\{\delta\theta^e\}_{e\in\Ead}}$.

The locality constraint in the training objective \Cref{eq:arm_relaxed} enforces $\delta\theta^e$ to remain close to the shared $\theta^c$ solutions. 
It plays several roles.
First, it fosters fast adaptation by acting as a constraint over $\theta^c\in\mathbb{R}^{d_\theta}$ during training \sut minimas $\smash{\{{\theta^e}^\star\}_{e \in \E}}$ are in a neighborhood of $\theta^c$ \ie can be reached from $\smash{\theta^c}$ with few update steps.
Second, it constrains the hypothesis space at fixed $\theta^c$. Under some assumptions, it can simplify the resolution of the optimization problem \wrt $\delta\theta^e$ by turning optimization to a quadratic convex problem with an unique solution.
We show this property for our solution in \Cref{prop:optimal_context}.
The positive effects of this constraint will be illustrated on an ODE system in \Cref{subsec:model}.

\subsection{Context-Informed Hypernetwork}
\label{subsec:model}
\begin{figure}[t!]
    \centering
    \includegraphics[width=\linewidth]{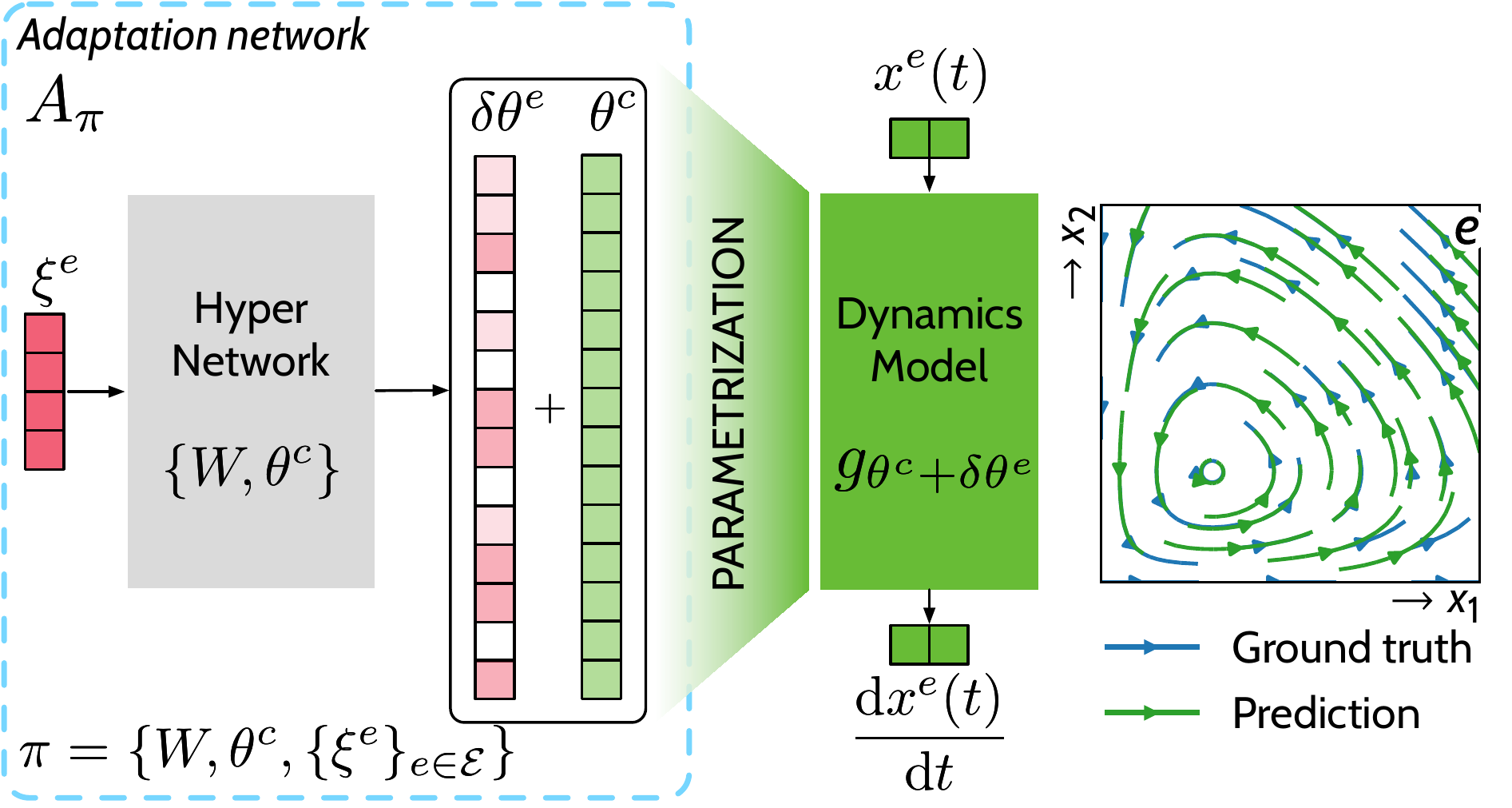}
    \vspace{-1.5em}
    \caption{Context-Informed Dynamics Adaptation (\ours).}%
    \label{fig:illus_coda}%
    \vspace{-1em}
\end{figure}
\Cref{eq:arm_relaxed} involves learning $\delta\theta^e$ for each environment.
For adaptation, $\delta\theta^e$ should be inferred from few observations of the new environment. 
Learning such high-dimensional parameters is prone to over-fitting, especially in low data regimes. 
We propose a hypernetwork-based solution (\Cref{fig:illus_coda}) to solve efficiently this problem.
It operates on a low-dimensional space, yields fixed-cost adaptation and shares efficiently information across environments.

\paragraph{Formulation}
We estimate $\delta\theta^e$ through a linear mapping of conditioning information, called context, learned from $\D^e$ and denoted $\xi^e\in\mathbb{R}^{d_\xi}$.
$W=(W_1, \cdots, W_{d_\xi})\in\mathbb{R}^{d_\theta \times d_\xi}$ is the weight matrix of the linear decoder \sut
\begin{equation}
    A_{\pi}(\D^e) \triangleq \theta^c + W\xi^e,\quad\pi\triangleq\{W, \theta^c, \{\xi^e\}_{e\in\E}\} \label{eq:adaptation_rule_ctx}
\end{equation}
$W$ is shared across environments and defines a low-dimensional subspace $\W\triangleq\Span(W_1, ..., W_{d_\xi})$, of dimension at most $d_\xi$, to which the search space of $\delta\theta^e$ is restricted.
$\xi^e$ is specific to each environment and can be interpreted as learning rates along the rows of $W$.
In our experiments, $d_\xi \ll d_\theta$ is small, at most 2. 
Thus, \textit{adaptation to new environments only requires to learn very few parameters, which define a completely new dynamics model $g$}.

$A_\pi$ corresponds to an affine mapping of $\xi^e$ parametrized by $\{W, \theta^c\}$, \aka a linear hypernetwork. 
Note that hypernetworks \citep{Ha2017} have been designed to handle single-environment problems and learn a separate context per layer. Our formalism involves multiple environments and defines a context per environment for all layers of $g$.

Linearity of the hypernetwork is not restrictive as contexts are directly learned through an inverse problem detailed in \cref{eq:coda-tr,eq:coda-ts}, \sut expressivity is similar to a nonlinear hypernetwork with a final linear activation.

\paragraph{Objectives}
We derive the training and adaptation objectives by inserting \Cref{eq:adaptation_rule_ctx} into \Cref{eq:arm_relaxed}.
For training, both contexts and hypernetwork are learned with \Cref{eq:coda-tr}:
\begin{equation}
   {\min_{\theta^c, W, \{\xi^e\}_{e\in\Etr}} \sum_{e\in\Etr}\! \Big(\Loss(\theta^c + W\xi^e, \D^e)+\lambda \|W \xi^e\|^2\Big)} \label{eq:coda-tr}%
\end{equation}
After training, $\theta^c$ is kept fixed and for adaptation to a new environment, only the context vector $\xi^e$  is learned with:
\begin{equation}
    {\min_{\{\xi^e\}_{e\in\Ead}}~ \sum_{e\in \Ead} \Big(\Loss(\theta^c + W\xi^e, \D^e)+\lambda \|W \xi^e\|^2\Big)} \label{eq:coda-ts}
\end{equation}
Implementation of \cref{eq:coda-tr,eq:coda-ts} is detailed in \Cref{sec:implementation}.
We apply gradient descent.
In \Cref{prop:optimal_context}, we show for $\|\cdot\|=\ell_2$, that \Cref{eq:coda-ts} admits an unique solution, recovered from initialization at $\vzero$ with a single preconditioned gradient step, projected onto subspace $\W$ defined by $W$.
\begin{proposition}[Proof in \Cref{app:proof}]\label{prop:optimal_context}
   Given $\{\theta^c,W\}$ fixed, if $\|\cdot\|{}=\ell_2$, then \Cref{eq:coda-ts} is quadratic. 
   If $\lambda^\prime W^\top W$ or $\bar{H}^e(\theta^c)=W^\top \nabla^2_\theta \Loss(\theta^c, \D^e) W$ are invertible then $\bar{H}^e(\theta^c)+\lambda^\prime W^\top W$ is invertible except for a finite number of $\lambda^\prime$ values.
   The problem in \Cref{eq:coda-ts} is then also convex and admits an unique solution, $\{{\xi^e}^\star\}_{e\in\Ead}$.
   With $\lambda^\prime \triangleq 2 \lambda$,
    \begin{equation}
        \smash{{\xi^e}^* = -\Big({\bar{H}^e(\theta^c)+\lambda^\prime W^\top W}\Big)^{-1} W^\top \nabla_\theta\Loss(\theta^c, \D^e)} \label{eq:optimal_context}
    \end{equation}
\end{proposition}
\paragraph{Interpretation}
We now interpret \ours by visualizing its loss landscape in \Cref{fig:loss_coda} and comparing it to ERM's loss landscape in \Cref{fig:loss_erm}. 
We use the package in \citet{Li2018b} to plot loss landscapes around $\theta^c$ and consider the \textit{Lotka-Volterra} system, described in \Cref{subsec:systems}.

In \Cref{fig:loss_coda}, loss values of \ours are projected onto subspace $\W$, where $d_\xi=2$.
We make three observations.
First, across environments, the loss is smooth and has a single minimum around $\theta^c$.
Second, the local optimum of the loss is close to $\theta^c$ across environments.
Finally, the minimal loss value on $\W$ around $\theta^c$ is low across environments.
The two first properties were discussed in \Cref{subsec:locality_optim} and are a direct consequence of the locality constraint on $\W$. 
When $\|\cdot\|=\ell_2$, it makes the optimization problem in \Cref{eq:coda-tr} quadratic \wrt $\xi^e$ and convex under invertibility of $\bar{H}^e(\theta^c)+\lambda^\prime W^\top W$ as detailed in \Cref{prop:optimal_context}.
We provided in \Cref{eq:optimal_context} the closed form expression of the solution.
It also imposes small $\|\xi^e\|$ \sut when minimizing the loss in \Cref{eq:coda-tr}, $\theta^c$ remains close to local optimas of all training environments. 
The final observation illustrates that \ours finds a subspace $\W$ with environment-specific parameters of low loss values \ie low-rank adaptation performs well.

In \Cref{fig:loss_erm}, loss values of ERM are projected onto the span of the two principal gradient directions. 
We observe that, unlike \ours, ERM does not find low loss values.
Indeed, it aims at finding $\theta^c$ with good performance across environments, thus cannot model several dynamics. 
\begin{figure}[t]
    \centering
    \subfloat[\ours's loss projected onto $\W=\Span(W_1,W_2)$.\label{fig:loss_coda}]{\includegraphics[width=0.9\linewidth]{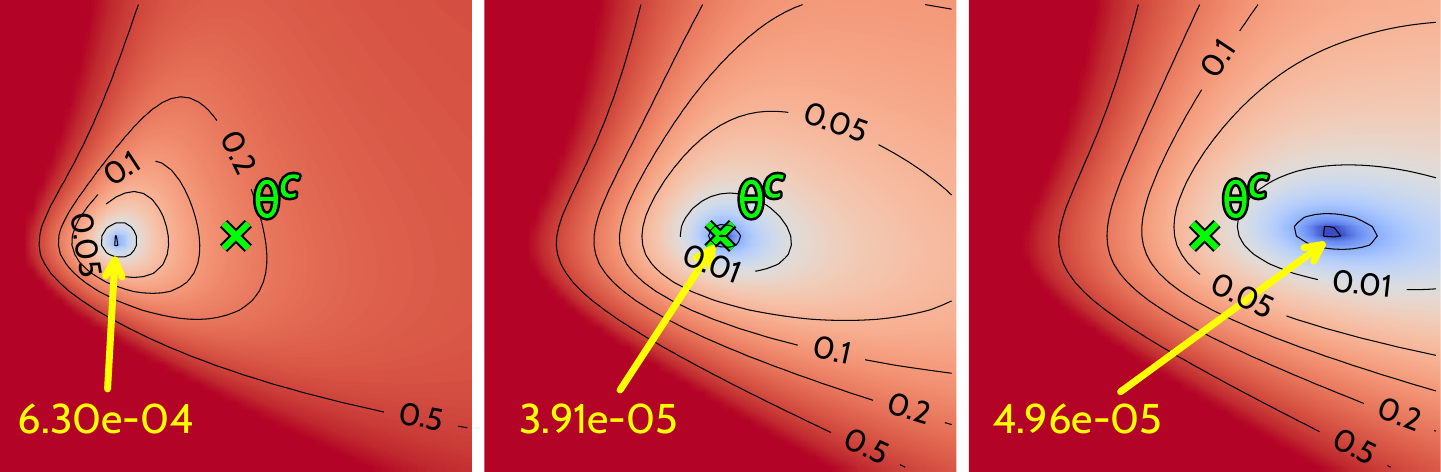}}
    \vspace{0.5em}
    
    \subfloat[ERM's loss projected onto the span of the two principal gradient directions computed with SVD.\label{fig:loss_erm}]{\includegraphics[width=0.9\linewidth]{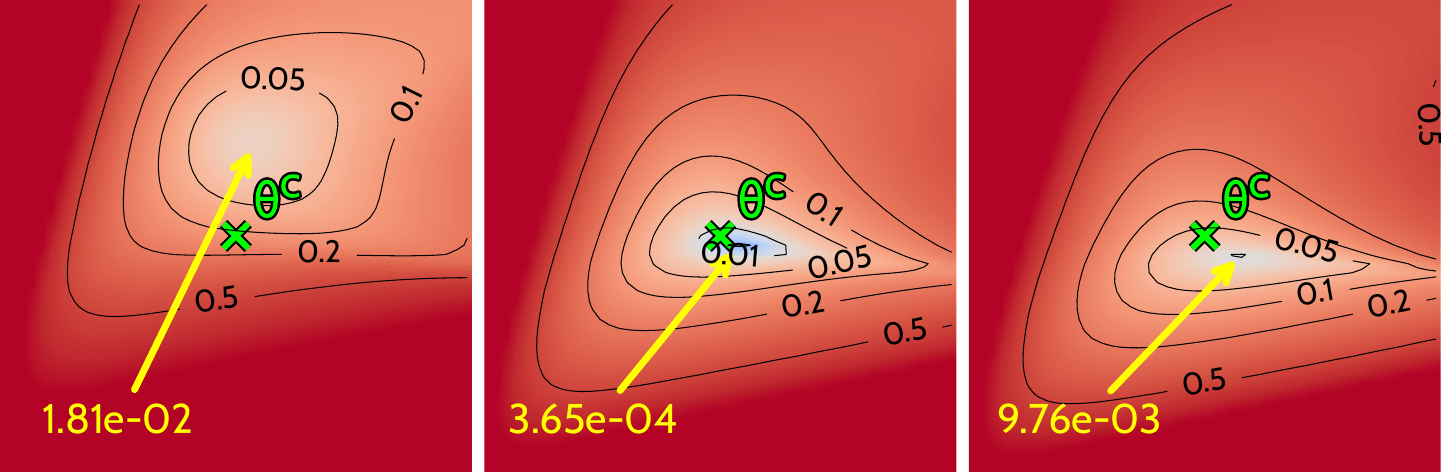}}
    \caption{Loss landscapes centered in $\theta^c$, marked with {\color{ForestGreen}\textbf{\texttimes}}, for 3 environments on the \textit{Lotka-Volterra} ODE. 
    $\forall e$, {\color{goldenrod}\textbf{\textrightarrow}} points to the local optimum ${\theta^e}^\star$ with loss value reported in {\color{goldenrod}yellow}.
    }
    \label{fig:loss_landscape}
    \vskip -0.1in
\end{figure}

\subsection{Validity for Dynamical Systems}
\label{subsec:validity}
We further motivate low-rank decoding in our context-informed hypernetwork approach by providing some evidence that gradients at $\theta^c$ across environments define a low-dimensional subspace.
We consider the loss $\Loss$ in \Cref{eq:Le} and define the gradient subspace in \Cref{def:grad_dir}.
\begin{definition}[Gradient directions]
    With $\Loss$ in \Cref{eq:Le}, $\forall \theta^c \in\mathbb{R}^{d_\theta}$ parametrizing a dynamics model $g_{\theta^c}$, the subspace generated by gradient directions at $\theta^c$ across environments $\E$ is denoted $\G_{\theta^c} \triangleq \Span(\{\nabla_{\theta} \Loss(\theta^c, \D^e)\}_{e \in \E})$. \label{def:grad_dir}
\end{definition}
We show, in \Cref{prop:lr_linearit}, low-dimensionality of $\G_{\theta^c}$ for linearly parametrized systems.
\begin{proposition}[Low-rank under linearity. Proof in \Cref{app:proof}]
    Given a class of linearly parametrized dynamics $\F$ with $d_p$ varying parameters, $\forall \theta^c\!\in\!\mathbb{R}^{d_\theta}\!$, subspace $\G_{\theta^c}\!$ in \Cref{def:grad_dir} is low-dimensional and $\dim(\G_{\theta^c}\!)\!\leq\!d_p\!\ll\!d_{\theta}$.\label{prop:lr_linearit}%
\end{proposition}
The linearity assumption is not restrictive as it is present in a wide variety of real-world systems \eg Burger or Korteweg–De Vries PDE \citep{raissi2019}, convection-diffusion \citep{Long2018}, wave and reaction diffusion equations \citep{Yin2021Aphynity} \etc 

Under nonlinearity, we do not have the same theoretical guarantee, yet, we show empirically in \Cref{app:low-rank-nonlinear} that low-dimensionality of parameters of the dynamics model still holds for several systems.
This property is comforted by recent work that highlighted that gradients are low-rank throughout optimization in single-domain settings, \ie that the solution space is low-dimensional \citep{gur-ari2019,Li2018,Li2018b}.
In the same spirit as \ours, this property was leveraged to design efficient solutions to the learning problems \citep{frankle2018, Vogels2019}.

\subsection{Benefits of \ours}
\label{subsec:discussion}
We highlight the benefits of \ours.
\ours is a general time-continuous framework that can be used with any approximator $g_\theta$ of the derivative \Cref{eq:Le}.
It can be trained with a given temporal resolution and tested on another; it handles irregularly-sampled sequences.
The choice of the approximator $g_\theta$ defines the ability to handle different spatial resolutions for PDEs, as further detailed in \Cref{subsec:implem_coda}.

Compared to related adaptation methods, \ours presents several advantages.
First, as detailed in \Cref{app:adaptation_rule}, the adaptation rule in \Cref{eq:adaptation_rule} is similar to the one used in gradient-based meta-learning; yet, our first order joint optimization problem in \Cref{eq:arm_relaxed} simplifies the complex bi-level optimization problem \citep{Antoniou2018}.
Second, \ours introduces the two key properties of locality constraint and low-rank adaptation which guarantee efficient adaptation to new environments as discussed in \Cref{subsec:model}.
Third, it generalizes contextual meta-learning methods \citep{Garnelo2018,ZintgrafSKHW2019}, which also perform low-rank adaptation, via the hypernetwork decoder (details in \Cref{app:conditioning_contextual}). 
Our decoder learns complex environment-conditional dynamics models while controlling their complexity.
Finally, \ours learns context vectors through an inverse problem as \citet{ZintgrafSKHW2019}.
This decoder-only strategy is particularly efficient and flexible in our setting. 
An alternative is to infer them via a learned encoder of $\D^e$ as \citet{Garnelo2018}. 
Yet, the latter was observed to underfit \citep{Kim2018}, requiring extensive tuning of the encoder and decoder architecture. 
Overall, \ours is easy to implement and maintains expressivity with a linear decoder.

\section{Framework Implementation}
\label{sec:implementation}
We detail how to perform trajectory-based learning with our framework and describe two instantiations of the locality constraint. 
We detail the corresponding  pseudo-code.

\paragraph{Trajectory-Based Formulation}
\label{subsec:integral_formulation}
As derivatives in \Cref{eq:Le} are not directly observed, we use in practice for training a trajectory-based formulation of \Cref{eq:Le}.
We consider a set of $N$ trajectories, $\D^e$.
Each trajectory is discretized over a uniform temporal and spatial grid and includes $\frac{T}{\Delta t} \left(\frac{S}{\Delta s}\right)^{d_s}$ states, where $d_s$ is the spatial dimension for PDEs and $d_s=0$ for ODEs.
$\Delta t, \Delta s$ are the temporal and spatial resolutions and $T, S$ the temporal horizon and spatial grid size.

Our loss writes as:
\begingroup
\setlength\abovedisplayskip{3.5pt}
\setlength\belowdisplayskip{3.5pt}
\begin{align}
    &\Loss(\theta, \D^e) = \sum_{i=1}^N \sum_{j=1}^{(\nicefrac{S}{\Delta s})^{d_s}}\sum_{k=1}^{\nicefrac{T}{\Delta t}}\left\|x^{e,i}(t_k, s_j)-\tilde{x}^{e,i}(t_k, s_j)\right\|^{2}_2 \notag\\[-4pt]
   &{\text{where~}\tilde{x}^{e,i}(t_k) = x^{e,i}_0 + \int_{t_{0}}^{t_{k}} g_\theta\left(\tilde{x}^{e,i}(\tau)\right) \diff \tau}\label{eq:Le_int}
\end{align}
\endgroup
$x^{e,i}(t_k, s_j)$ is the state value in the $i^{th}$ trajectory from environment $e$ at the $j^{th}$ spatial coordinate $s_j$ and time $t_k\triangleq k \Delta t$. 
$x^{e,i}(t) \triangleq [x(t, s_1), \cdots, x(t, s_{\left(\nicefrac{S}{\Delta s}\right)^{d_s}})]^\top$ is the state vector in the $i^{th}$ trajectory from environment $e$ over the spatial domain at time $t$ and $x^{e,i}_0$ is the corresponding initial condition.
To compute $\tilde{x}^{e,i}(t_k)$, we apply for integration a numerical solver \citep{Hairer2000} as detailed later.

\paragraph{Locality Constraint}
\label{subsec:locality_constraint}
Instead of penalizing $\lambda \|W\xi^e\|^2$ in \Cref{eq:coda-tr}, we found it more efficient to penalize separately $W$ and $\xi^e$. 
We thus introduce the following regularization:
\begin{equation}
    R(W, \xi^e) \triangleq \lambda_\xi \|\xi^e\|_2^2 + \lambda_\Omega \Omega(W) \label{eq:reg}
\end{equation}
It involves hyperparameters $\lambda_\xi, \lambda_\Omega$ and a norm $\Omega(W)$ which depends on the choice of $\|\cdot\|$ in \Cref{eq:arm_relaxed}.
Minimizing $R(W,\xi^e)$ minimizes an upper-bound to $\|\cdot\|$, derived in \cref{app:locality_upper_bound} for the two considered variations of $\|\cdot\|$:
\begin{itemize}
    \item \oursltwo sets $\|\cdot\|{}\triangleq \ell_2(\cdot)$ and $\Omega{}\triangleq \ell_2^2$, constraining $W \xi^e$ to a sphere.
    \item \ourslone sets $\|\cdot\|{}\triangleq \ell_{1}(\cdot)$ and $\Omega=\ell_{1,2}$ over rows \ie $\Omega(W){}\triangleq \sum_{i=1}^{d_\theta} \|W_{i,:}\|_2$ to induce sparsity and find most important parameters for adaptation.
$\smash{\ell_{1,2}}$ constrains $\W$ to be axis-aligned; then the number of solutions is finite as $\dim(\W)$ is finite.
\end{itemize}

\paragraph{Pseudo-Code}
\label{subsec:pseudo-code}
We solve \Cref{eq:coda-tr} for training and \Cref{eq:coda-ts} for adaptation using \cref{eq:Le_int,eq:reg} and \Cref{alg:training_adaptation}. 
We back-propagate through the solver with torchdiffeq \citep{Chen_torchdiffeq_2021} and apply exponential Scheduled Sampling \citep{Goyal2016} to stabilize training. 
We provide our code at \url{https://github.com/yuan-yin/CoDA}.

\begin{figure}
\vskip -0.1in
\begin{algorithm}[H]
    \caption{\ours Pseudo-code} \label{alg:training_adaptation}
    \begin{algorithmic}
        \STATE \hspace*{-1.25em} \textit{\underline{Training}:}
          \REQUIRE $\Etr\subset\E$, $\{\D^{e_{\text{tr}}}\}_{e_{\text{tr}}\in\Etr}$ with $\forall e_{\text{tr}}\in\Etr, \#\D^{e_{\text{tr}}}=N_{\text{tr}}$;
          \STATE $\pi=\{W, \theta^c, \{\xi^{e_{\text{tr}}}\}_{e_{\text{tr}} \in \Etr}\}$ where $W\in\mathbb{R}^{d_\theta\times d_\xi}, \theta^c\in\mathbb{R}^{d_\theta}$ randomly initialized and $\forall e_{\text{tr}}\in\Etr, \xi^{e_{\text{tr}}}=\vzero\in\mathbb{R}^{d_\xi}$.
          \LOOP
            \STATE $\smash{\pi \leftarrow \pi - \eta \nabla_{\pi}\!\Big(\!\sum\limits_{e_{\text{tr}}\in\Etr}\!\Loss(\theta^c + W \xi^{e_{\text{tr}}}, \D^{e_{\text{tr}}})\!+\!R(W, \xi^{e_{\text{tr}}})\!\Big)}$ 
          \ENDLOOP
          \STATE \hspace*{-1.25em} \textit{\underline{Adaptation}:}
          \REQUIRE $e_{\text{ad}}\in\Ead$; $\D^{e_{\text{ad}}}$ with $\#\D^{e_{\text{ad}}}=N_{\text{ad}}$;
          \STATE Trained $W\in\mathbb{R}^{d_\theta\times d_\xi}, \theta^c\in\mathbb{R}^{d_\theta}$ and $\xi^{e_{\text{ad}}}=\vzero\in\mathbb{R}^{d_\xi}$.
          \LOOP 
          \STATE $\smash[t]{\xi^{e_{\text{ad}}}\!\leftarrow\!\xi^{e_{\text{ad}}}\!- \eta \nabla_{\!\xi^{e_{\text{ad}}}\!}\Big(\!\Loss(\theta^c + W \xi^{e_{\text{ad}}}, \D^{e_{\text{ad}}})+R(W, \xi^{e_{\text{ad}}})\!\Big)}$
          \ENDLOOP
    \end{algorithmic}
\end{algorithm}
\vskip -0.3in
\end{figure}

\begin{table*}[t]
    \centering
    \vspace{-0.12in}
    \caption{Test MSE ($\downarrow$) in training environments $\Etr$ (\textit{In-Domain}), new environments $\Ead$ (\textit{Adaptation}). Best in \textbf{bold}; second \underline{underlined}.} \label{tab:results}
    \scshape 
    \vskip 0.1in
    \resizebox{\linewidth}{!}{
    \begin{tabular}{@{}lllcllcllcll@{}}
         \toprule
          & \multicolumn{2}{c}{\texttt{LV} ($\times 10^{-5}$)}  && \multicolumn{2}{c}{\texttt{GO} ($\times 10^{-4}$)} && \multicolumn{2}{c}{\texttt{GS}  ($\times 10^{-3}$)} && \multicolumn{2}{c}{\texttt{NS} ($\times 10^{-4}$)} \\
         \cmidrule(lr){2-3} \cmidrule(lr){5-6} \cmidrule(lr){8-9} \cmidrule(lr){11-12} & In-domain & Adaptation && In-domain & Adaptation && In-domain & Adaptation && In-domain & Adaptation \\
         \midrule
         MAML & 60.3±1.3 & 3150±940 && 57.3±2.1 & 1081±62 && 3.67±0.53 & 2.25±0.39 && 68.0±8.0 & 51.1±4.0 \\
         ANIL & 381±76 & 4570±2390 && 74.5±11.5 & 1688±226 &&   5.01±0.80 & 3.95±0.11 && 61.7±4.3 & 48.6±3.2 \\
         Meta-SGD & 32.7±12.6 & 7220±4580 && 42.3±6.9 & 1573±413 && 2.85±0.54 & 2.68±0.20 && 53.9±28.1 & 44.3±27.1\\
         LEADS & 3.70±0.27 & 47.61±12.47 && 31.4±3.3 & 113.8±41.5 && 2.90±0.76 & 1.36±0.43 && 14.0±1.55 & 28.6±7.23 \\ 
         CAVIA-FiLM & 4.38±1.15 & 8.41±3.20 && 4.44±1.46 & 3.87±1.28 && 2.81±1.15 & 1.43±1.07 && 23.2±12.1 & 22.6±9.88 \\
         CAVIA-Concat & 2.43±0.66 & 6.26±0.77 && 5.09±0.35 & 2.37±0.23 && 2.67±0.48 & 1.62±0.85 && 25.5±6.31 & 26.0±8.24\\
         \ours-$\ell_2$ & \underline{1.52}±0.08 & \underline{1.82}±0.24 && \underline{2.45}±0.38 & \underline{1.98}±0.06 && \underline{1.01}±0.15 & \underline{0.77}±0.10 && \underline{9.40}±1.13 & \underline{10.3}±1.48 \\ 
         \ours-$\ell_1$ & \textbf{1.35}±0.22 & \textbf{1.24}±0.20 && \textbf{2.20}±0.26 & \textbf{1.86}±0.29 && \textbf{0.90}±0.057 & \textbf{0.74}±0.10 && \textbf{8.35}±1.71 & \textbf{9.65}±1.37 \\ 
         \bottomrule
    \end{tabular}}
    \vskip -0.15in
\end{table*}
\section{Experiments}
\label{sec:experiments}
We validate our approach on four classes of challenging nonlinear temporal and spatiotemporal physical dynamics, representative of various fields \eg chemistry, biology and fluid dynamics.
We evaluate in-domain and adaptation prediction performance and compare them to related baselines.
We also investigate how learned context vectors can be used for system parameter estimation. 
We consider a few-shot adaptation setting where only few trajectories ($N_{\text{ad}}$) are available at adaptation time on new environments.

\subsection{Dynamical Systems}
\label{subsec:systems}
We consider four ODEs and PDEs described in \Cref{app:dynsys}.
ODEs include \textit{Lotka-Volterra} (\texttt{LV}, \citealp{Lotka1925}) and \textit{Glycolitic-Oscillator} (\texttt{GO}, \citealp{Daniels2015}), modelling respectively predator-prey interactions and the dynamics of yeast glycolysis.
PDEs are defined over a 2D spatial domain and include  \textit{Gray-Scott} (\texttt{GS}, \citealp{Pearson1993}), a reaction-diffusion system with complex spatiotemporal patterns and the challenging \textit{Navier-Stokes} system (\texttt{NS}, \citealp{stokes1851effect}) for incompressible flows.
All systems are nonlinear \wrt system states and all but \texttt{GO} are linearly parametrized.
The analysis in \Cref{subsec:validity} covers all systems but \texttt{GO}. 
Experiments on the latter show that \ours also extends to nonlinearly parametrized systems.

\subsection{Experimental Setting}
We consider forecasting: only the initial condition is used for prediction.
We perform two types of evaluation: in-domain generalization on $\Etr$ \textit{(In-domain)} and out-of-domain adaptation to new environments $\Ead$ (\textit{Adaptation}).
Each environment $e\in\E$ is defined by system parameters and $p^e \in \mathbb{R}^{d_p}$ denotes those that vary across $\E$.
$d_p$ represents the degrees of variations in $\F$; $d_p=2$ for \texttt{LV}, \texttt{GO}, \texttt{GS} and $d_p=1$ for \texttt{NS}.
\Cref{app:dynsys} defines for each system the number of training and adaptation environments ($\#\Etr$ and $\#\Ead$) and the corresponding parameters.
\Cref{app:dynsys} also reports the number of trajectories $N_{\text{tr}}$ per training environment in $\Etr$ and the distribution $p(X_0)$ from which are sampled all initial conditions (including adaptation and evaluation initial conditions).
For \textit{Adaptation}, we consider $N_{\text{ad}}=1$ trajectory per new environment in $\Ead$ to infer the context vector with \Cref{eq:coda-ts}.
We consider more trajectories per adaptation environment in \Cref{subsec:sample_efficiency}.

Evaluation is performed on 32 new test trajectories per environment.
We report, in our tables, mean and standard deviation of Mean Squared Error (MSE) across test trajectories (\Cref{eq:Le_int}) over four different seeds.
We report, in our figures, Mean Absolute Percentage Error (MAPE) in $\%$ over trajectories, as it allows to better compare performance across environments and systems.
We define MAPE$(z,y)$ between a $d$-dimensional input $z$ and target $y$ as $\frac{1}{d}\sum_{j=1\ldots d: y_j\neq 0}\frac{|z_j-y_j|}{|y_j|}$.
Over a trajectory, it extends into $\int_{t\in I}\text{MAPE}(\tilde{x}(t), x(t))\diff t$, with $\tilde{x}$ defined in \Cref{eq:Le_int}.

\subsection{Implementation of \ours}
\label{subsec:implem_coda}
We used for $g_\theta$ MLPs for ODEs, a resolution-dependent ConvNet for \texttt{GS} and a resolution-agnostic FNO \cite{Li2021} for \texttt{NS} that can be used on new resolutions.
Architecture details are provided in \Cref{app:implem_hyperparam}.
We tuned $d_\xi$ and observed that $d_\xi=d_p$, the number of system parameters that vary across environments, performed best (\cf \Cref{subsec:ablation_study}).
We use Adam optimizer \cite{Kingma2015} for all datasets; RK4 solver for \texttt{LV}, \texttt{GS}, \texttt{GO} and Euler solver for \texttt{NS}. 
Optimization and regularization hyperparameters are detailed in \Cref{app:implem_hyperparam}.

\subsection{Baselines}
\label{subsec:baseline_archi}
We consider three families of baselines, compared in Appendix \Cref{fig:baselines} and detailed in \Cref{sec:related_work}.
First, Gradient-Based Meta-Learning (GBML) methods MAML \citep{Finn2017}, ANIL \citep{Rusu2018} and Meta-SGD \citep{Li2017}.
Second, the Multi-Task Learning method LEADS \citep{Yin2021LEADS}.
Finally, the contextual meta-learning method CAVIA \citep{ZintgrafSKHW2019}, with conditioning via concatenation (Concat) or linear modulation of final hidden features (FiLM, \citealp{Perez2018}).
All baselines are adapted to be dynamics-aware \ie time-continuous: they consider the loss in \Cref{eq:Le_int}, as \ours.
Moreover, they share the same architecture for $g_\theta$ as \ours.

\subsection{Generalization Results}
\label{sec:generalization_results}
\begin{figure}[t]
    \centering
    \vskip 0.1in
    \includegraphics[width=0.95\linewidth]{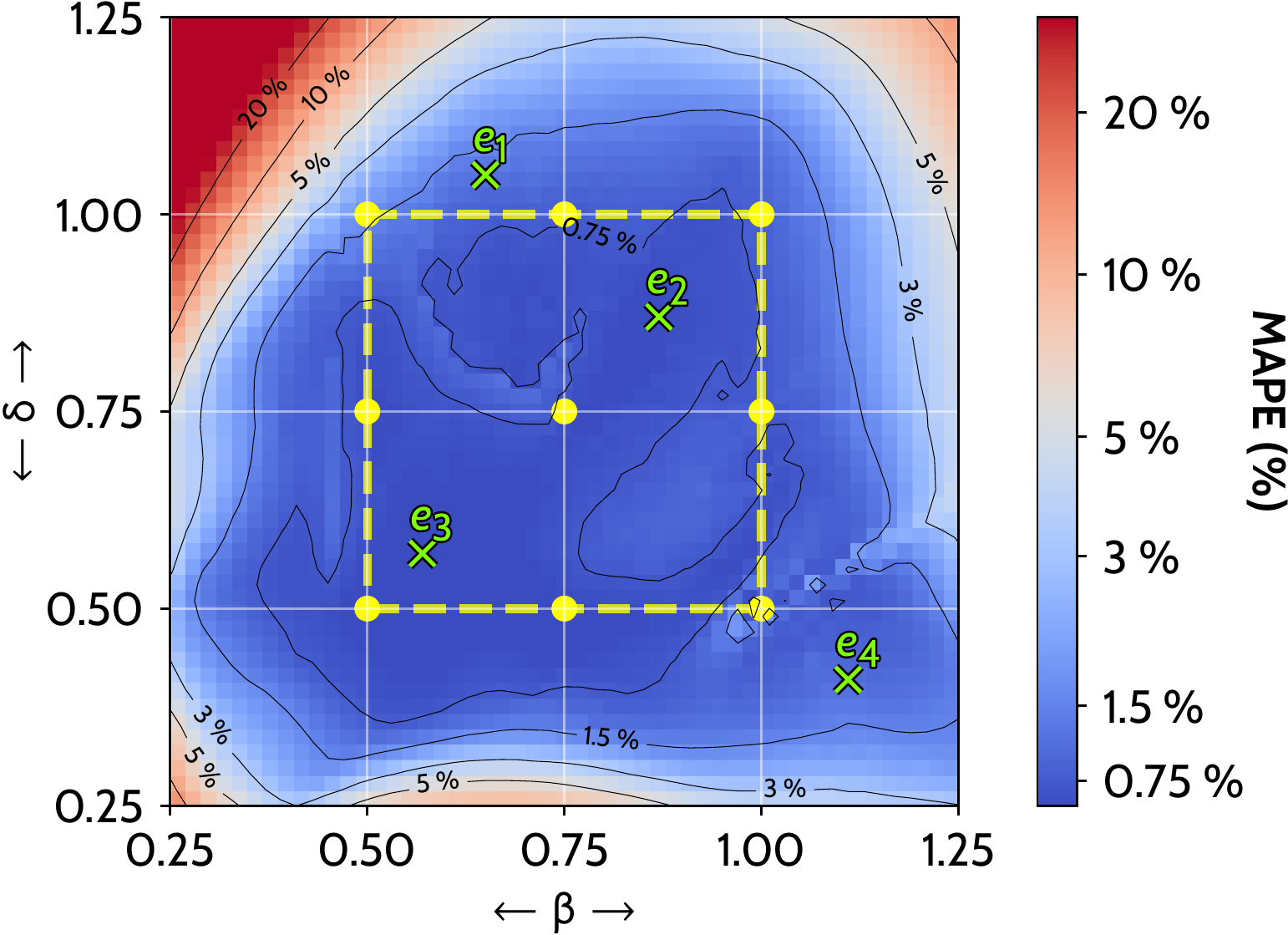}
    \vskip 0.1in
    \includegraphics[width=0.95\linewidth]{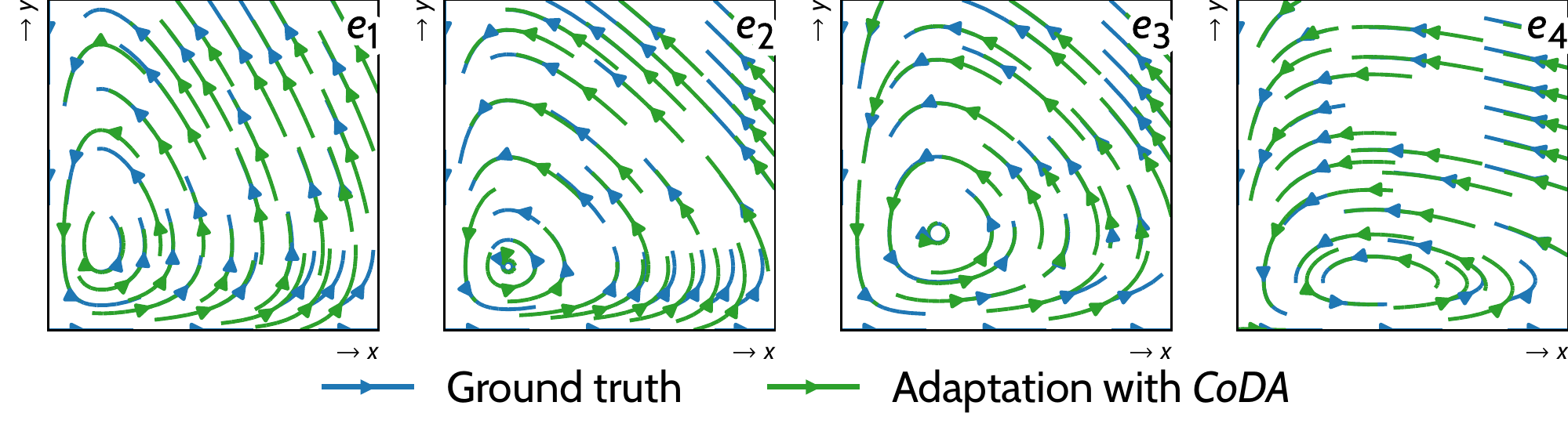}
    \vskip -0.05in
    \caption{\textit{Adaptation} results with \ourslone on \texttt{LV}. Parameters $(\beta, \delta)$ are sampled in $\smash{[0.25, 1.25]^2}$ on a $51\times 51$ uniform grid, leading to 2601 adaptation environments $\Ead$. {\color{goldenrod}$\bullet$} are training environments $\Etr$. We report MAPE ($\downarrow$) across $\Ead$ (top). On the bottom, we choose four of them (${\color{ForestGreen}\textbf{\texttimes}}$, $e_1$--$e_4$), to show the ground-truth ({\color{blue}blue}) and predicted ({\color{ForestGreen}green}) phase space portraits. $x, y$ are respectively the quantity of prey and predator in the system in \Cref{eq:lv}.} \label{fig:adaptation_context}
    %\vskip -2em
\end{figure}
\begin{figure}[t]\label{fig:ablation_dim}
    \centering
    \includegraphics[width=\linewidth]{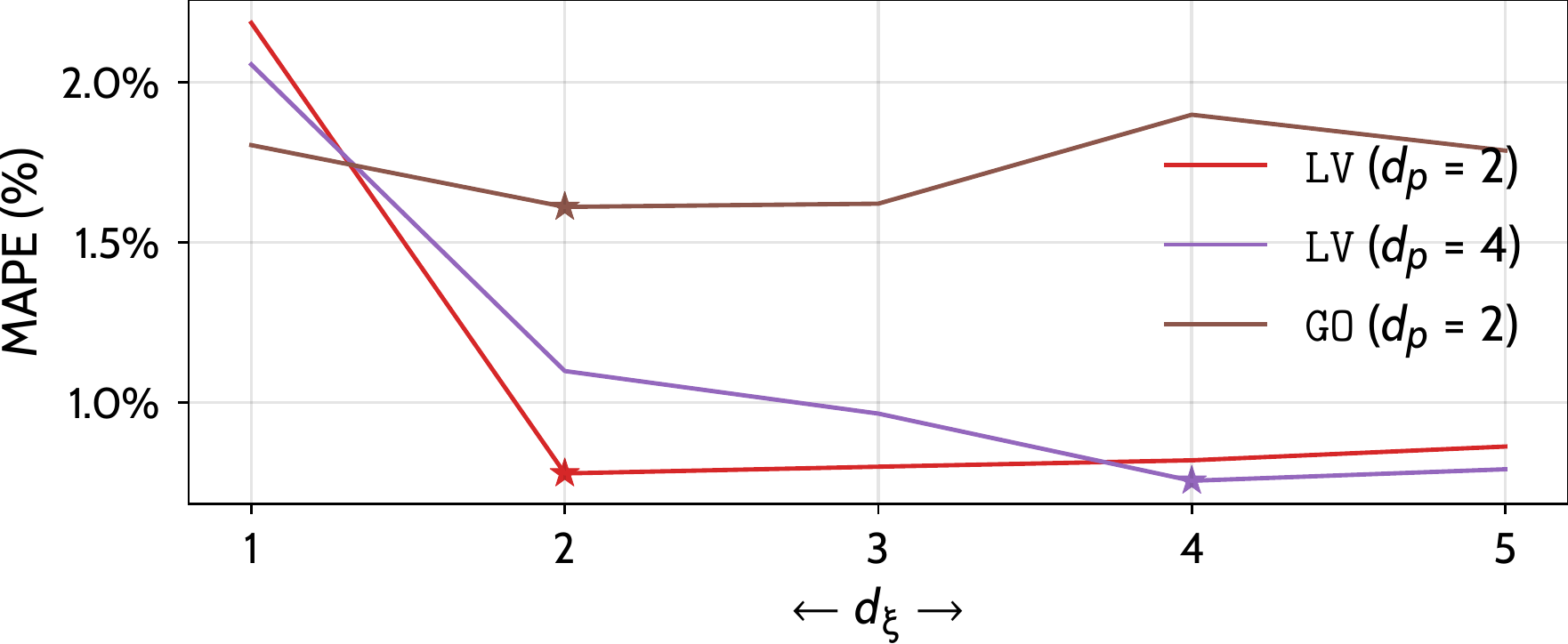}%
    \vskip -0.1in
    \caption{Dimension of the context vectors ($d_\xi$) and test \textit{In-Domain} MAPE ($\downarrow$) with \ours-$\ell_1$. ``$\star$'' is the smallest MAPE.} 
    \vskip -0.15in
\end{figure}
\begin{table}[t!]
    \centering
    \vskip -0.1in
    \caption{Locality and \textit{In-Domain} test MSE ($\downarrow$). Best in \textbf{bold}.} \label{table:ablation_locality}
    \vskip 0.1in
    \scshape
    \resizebox{\linewidth}{!}{
    \begin{tabular}{llllll}
         \toprule
         & \multicolumn{2}{c}{\texttt{LV} ($\times 10^{-5}$)} && \multicolumn{2}{c}{\texttt{GO} ($\times 10^{-4}$)} \\
         \cmidrule(lr){2-3} \cmidrule(lr){5-6} \ours & W/o $\ell_2$ & With $\ell_2$ && W/o $\ell_2$ & With $\ell_2$ \\
         \midrule
         Full & 2.28±0.29 & 1.52±0.08 && 2.98±0.71 & 2.45±0.38 \\
         FirstLayer & 2.25±0.29 & 2.41±0.23 && 2.38±0.71 & \textbf{2.12}±0.55 \\ 
         LastLayer & 1.86±0.24 & \textbf{1.27}±0.03 && 28.4±0.60 & 28.4±0.64 \\
         \bottomrule
    \end{tabular}}
    \vskip -0.15in
\end{table}
\begin{figure*}[t!]
    \vskip -0.1in
    \centering
    \subfloat[{\texttt{LV}: $(\beta, \delta) \in [0.25, 1.25]^2$}\label{fig:context_param}]{\includegraphics[height=1.55in,trim={0.3in 0 0 0}]{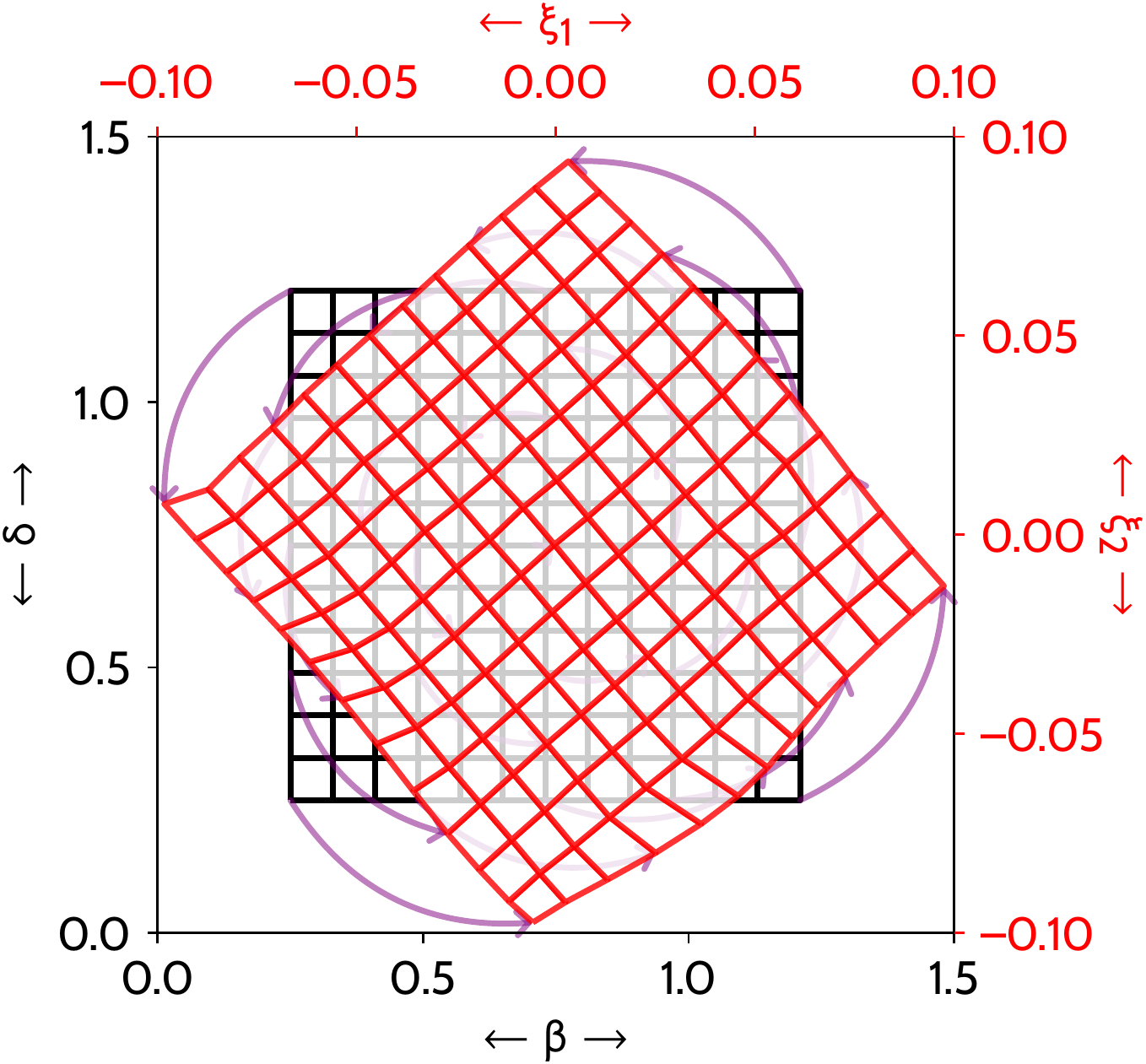}~
    \includegraphics[height=1.35in]{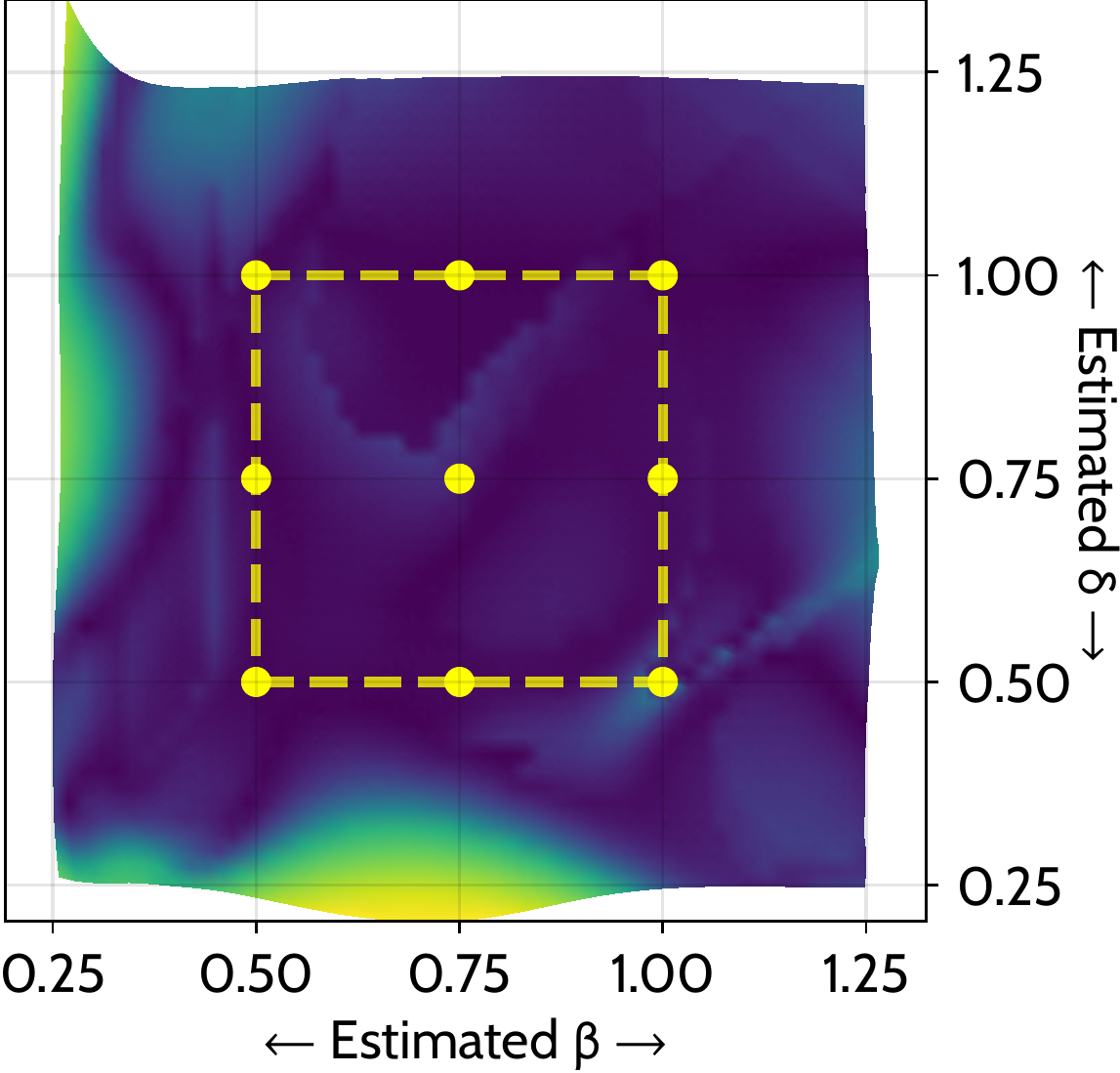} \vspace{-0.05in}}
    \subfloat[{\texttt{GS}: $F\in [2.25,4.35]\cdot 10^{-2}$, $k\in[5.6,6.4]\cdot10^{-2}$}\label{app:param_gs}]{\includegraphics[height=1.4in,trim={0.1in 0.1in 0.1in 0}, clip]{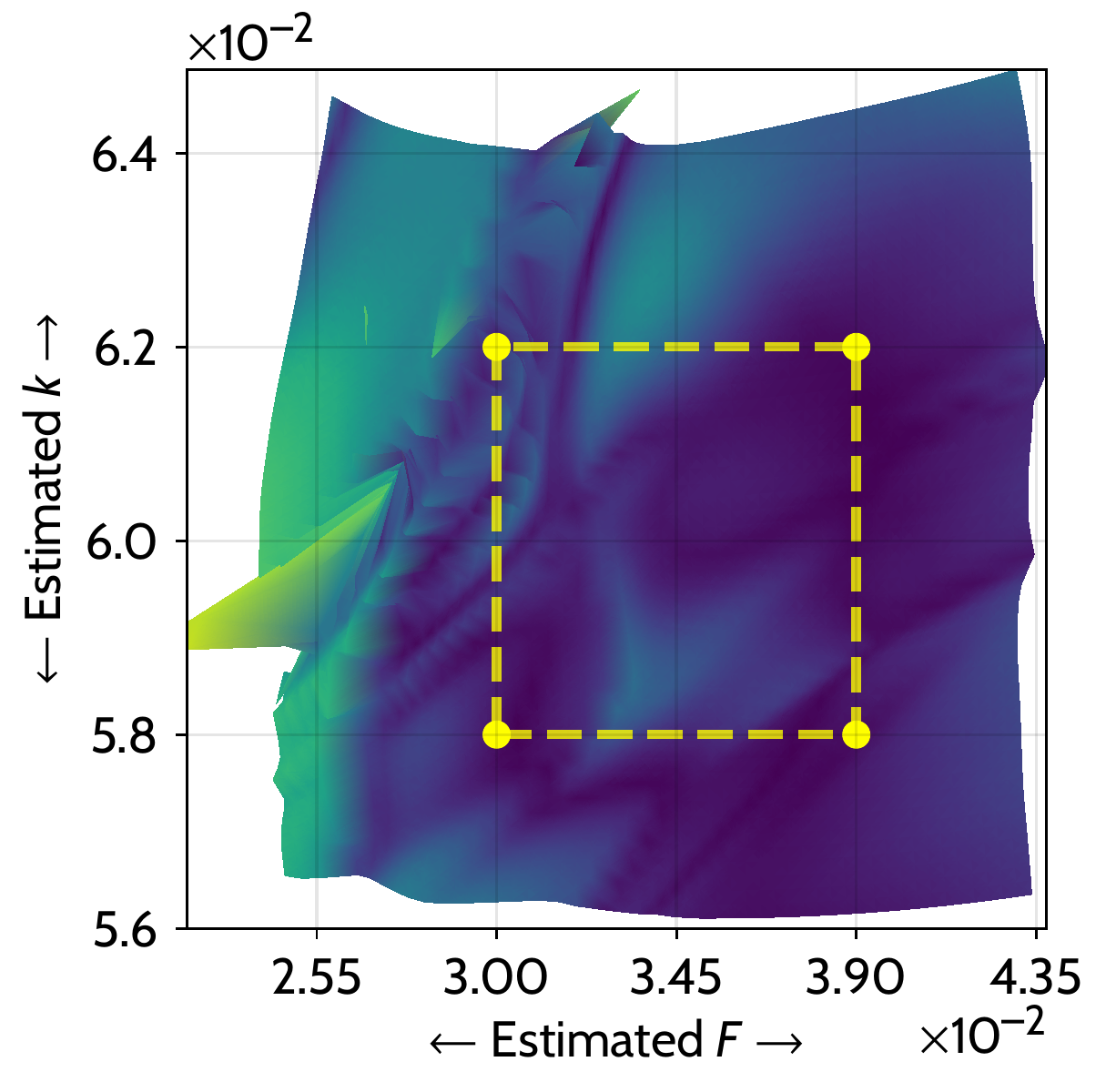}}
    %\vspace{0.05in}
    \subfloat[{\texttt{NS}: $\nu \in  [0.6, 1.4]\cdot 10^{-3}$}\label{fig:param_ns}]{\includegraphics[align=t,width=0.35\linewidth]{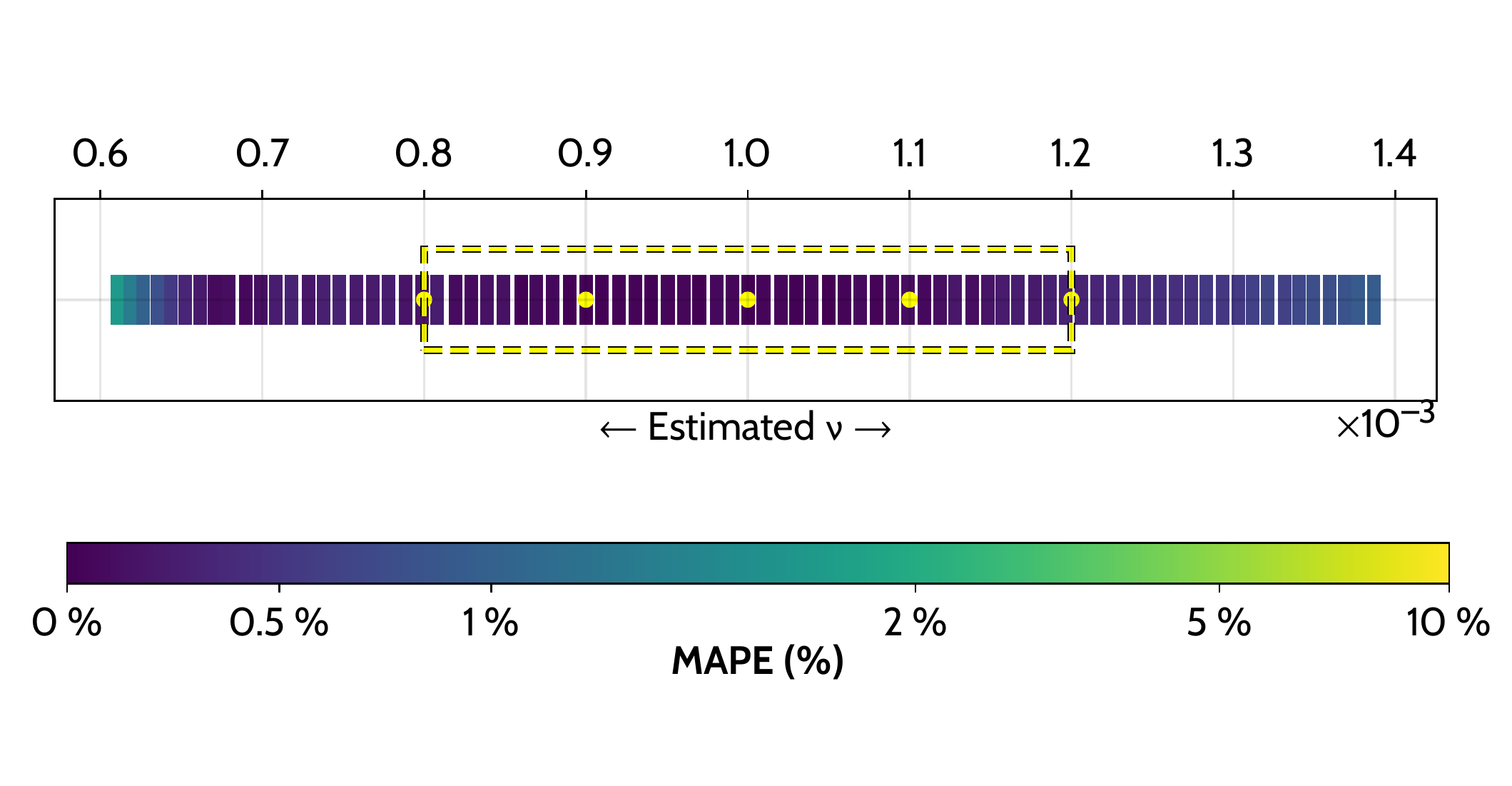}}
    \vskip -0.05in
    \caption{Parameter estimation with \ourslone in new adaptation environments on (a) \texttt{LV}, (b) \texttt{GS} and (c) \texttt{NS}. In (a), we visualize: on the left, context vectors $\xi$ ({\color{red}red}); on the right, true parameters $(\beta, \delta)$ (black). In (b) and (c), we visualize estimated parameters with corresponding estimation MAPE ($\downarrow$). {\color{goldenrod}$\bullet$} are training environments $\Etr$ with known parameters. {\color{goldenrod}-~-} delimits the convex hull of $\Etr$.} \label{fig:parameter_estimation}
    \vskip -0.15in
\end{figure*}
In \Cref{tab:results}, we observe that \ours improves significantly test MSE \wrt our baselines for both \textit{In-Domain} and \textit{Adaptation} settings.
For PDE systems and a given test trajectory, we visualize in \Cref{fig:gs_traj,fig:ns_traj} in \Cref{app:visualization} the predicted MSE by these models along the ground truth.
We also notice improvements for \ours over our baselines.
Across datasets, all baselines are subject to a drop in performance between \textit{In-Domain} and \textit{Adaptation} while \ours maintains remarkably the same level of performance in both cases.
In more details, GBML methods (MAML, ANIL, Meta-SGD) overfit on training \textit{In-Domain} data especially when data is scarce.
This is the case for ODEs which include less system states for training than PDEs. 
LEADS performs better than GBML but overfits for \textit{Adaptation} as it does not adapt efficiently.
CAVIA-Concat/FiLM perform better than GBML and LEADS, as they leverage a context, but are less expressive than \ours.
Both variations of \ours perform best as they combine the benefits of low-rank adaptation and locality constraint.
\ourslone is better than \oursltwo as it induces sparsity, further constraining the hypothesis space. 

We evaluate in \Cref{fig:adaptation_context} \ourslone on \texttt{LV} for \textit{Adaptation} over a wider range of adaptation environments ($\#\Ead=51\times 51=2601$). 
We report mean MAPE over $\Ead$ (top). 
We observe three regimes: inside the convex hull of training environments $\Etr$, MAPE is very low; outside the convex-hull, MAPE remains low in a neighborhood of $\Etr$; beyond this neighborhood, MAPE increases.
\ours thus generalizes efficiently in the neighborhood of training environments and degrades outside this neighborhood.
We plot reconstructed phase space portraits (bottom) on four selected environments and observe that the learned solution ({\color{ForestGreen}green}) closely follows the target trajectories ({\color{blue}blue}).

\subsection{Ablation Studies}
\label{subsec:ablation_study}
We perform two studies on \texttt{LV} and \texttt{GO}.
In a first study in \Cref{table:ablation_locality}, we evaluate the gains due to using $\ell_2$ locality constraint on \textit{In-Domain} evaluation.
On line 1 (Full), we observe that \oursltwo performs better than \ours without locality constraint.
Prior work perform adaptation only on the final layer with some performance improvements on classification or Hamiltonian system modelling \citep{Raghu2020,Chen2020}. 
In order to evaluate this strategy, we manually restrict hypernetwork-decoding to only one layer in the dynamics model $g_\theta$, either the first layer (line 2) or the last layer (line 3). 
We observe that the importance of the layer depends on the parametrization of the system: for \texttt{LV}, linearly parametrized, the last layer is better while for \texttt{GO}, nonlinearly parametrized, the first layer is better. 
\ourslone generalizes this idea by automatically selecting the useful adaptation subspace via $\ell_{1,2}$ regularization, offering a more flexible approach to induce sparsity. 

In a second study in \Cref{fig:ablation_dim}, we analyze the impact on MAPE of the dimension of context vectors $d_\xi$ for \ourslone.
We recall that $d_\xi$ upper-bounds the dimension of the adaptation subspace $\W$ and was cross-validated in \Cref{tab:results}.
In the following, $d_p$ is the number of parameters that vary across environments.
We illustrate the effect of the cross-validation on MAPE for $d_p=2$ on \texttt{LV} and \texttt{GO} as in \Cref{sec:generalization_results} and additionally for $d_p=4$ on \texttt{LV}.
We observe in \Cref{fig:ablation_dim} that the minimum of MAPE is reached for $d_\xi=d_p$ with two regimes: when $d_\xi<d_p$, performance decreases as some system dimensions cannot be learned; when $d_\xi > d_p$, performance degrades slightly as unnecessary directions of variations are added, increasing the hypothesis search space.
This study shows the validity of the low-rank assumption and illustrates how the unknown $d_p$ can be recovered through cross-validation.

\subsection{Sample Efficiency}
\label{subsec:sample_efficiency}
\begin{table}[t]
    \centering
    \scshape
    \caption{Test MSE $\times 10^{-5}$ ($\downarrow$) in new environments $\Ead$ (\textit{Adaptation}) on \textit{Lotka-Volterra}. Best for each setting in \textbf{bold}.}
    \vskip 0.1in
    \resizebox{0.8\linewidth}{!}{
    \begin{tabular}{lllll}
         \toprule
         & \multicolumn{3}{c}{Number of adaptation trajectories {\normalshape$N_{\text{ad}}$}} \\
         \cmidrule(lr){2-4}  & 1 & 5 & 10 \\
         \midrule
         MAML & 3150±940 & 239±16 & 173±10\\
         LEADS & 47.61±12.47 & 19.89±7.23 & 19.42±3.52\\
         \ours-$\ell_1$ & \textbf{1.24}±0.20 & \textbf{1.21}±0.18 & \textbf{1.20}±0.17\\
         \bottomrule
    \end{tabular}} \label{tab:sample_efficiency}
    \vskip -0.15in
\end{table}
\begin{table}
\vskip -0.05in
    \centering
    \scshape
    \caption{Parameter estimation MAPE ($\downarrow$) for \ours-$\ell_1$ on \texttt{LV} ($\#\Etr=9$), \texttt{GS} ($\#\Etr=4$) and \texttt{NS} ($\#\Etr=5$).} \label{tab:param_error_all}
    \vskip 0.1in
    \resizebox{\linewidth}{!}{
    \begin{tabular}{llllllll}
    \toprule
      & \multicolumn{2}{c}{In-convex-hull} && \multicolumn{2}{c}{Out-of-convex-hull} && Overall \\
     \cmidrule(lr){2-3} \cmidrule(lr){5-6} \cmidrule(lr){8-8} & MAPE (\%) & {\normalshape$\#\Ead$} && MAPE (\%) & {\normalshape$\#\Ead$} && MAPE (\%) \\
     \midrule
      \texttt{LV} & 0.15±0.11 & 625 && 0.73±1.33 & 1976 && 0.59±1.33 \\
      \texttt{GS} & 0.37±0.25 & 625 && 0.74±0.67 & 1976 && 0.65±0.62 \\
      \texttt{NS} & 0.10±0.08 & 40 && 0.51±0.35 & 41 && 0.30±0.33 \\
    \bottomrule
    \end{tabular}}
    \vskip -0.15in
\end{table}
We handled originally one-shot adaptation ($N_{\text{ad}}=1$), the most challenging setting.
We vary the number of adaptation trajectories $N_{\text{ad}}$ on \textit{Lotka-Volterra} in \Cref{tab:sample_efficiency}.
With more trajectories, performance improves significantly for MAML; moderately for LEADS; while it remains flat for \ours.
This highlights \ours's sample-efficiency and meta-overfitting for GBML \citep{Mishra2018}.

\subsection{Parameter Estimation}
\label{sec:qualitative_results}
We use \ours to perform parameter estimation, leveraging the links between learned context and system parameters. 
\subsubsection{Empirical observations}
In \Cref{fig:context_param} (left), we visualize on \texttt{LV} the learned context vectors $\xi^e$ ({\color{red} red}) and the system parameters $p^e$ (black), $\forall e\in\Etr\cup\Ead$.
We observe empirically a linear bijection between these two sets of vectors.
Such a correspondence being learned on the training environments, we can use the correspondence to verify if it still applies to new adaptation environments.
Said otherwise, we can check if our model is able to infer the true parameters for new environments.

We evaluate in \Cref{tab:param_error_all} the parameter estimation MAPE over \texttt{LV}, \texttt{GS} and \texttt{NS}.
\Cref{fig:parameter_estimation} displays estimated parameters along estimation MAPE.
Experimentally, we observe low MAPE inside and even outside the convex-hull of training environments. 
Thus, \ours identifies accurately the unknown system parameters with little supervision.

\subsubsection{Theoretical motivation}
We justify these empirical observations theoretically in \Cref{prop:context_bijectivity}  under the following conditions:
\begin{assumption} \label{ass:linear_dyn}
    The dynamics in $\F$ are linear \wrt inputs and system parameters.
\end{assumption}
\begin{assumption} \label{ass:linear_model}
    Dynamics model $g$, hypernet $A$ are linear.
\end{assumption} 
\begin{assumption} \label{ass:unicity_param}
    $\forall e\in\E$, parameters $p^e \in \mathbb{R}^{d_p}$ are unique.
\end{assumption} 
\begin{assumption} \label{ass:dim_code}
    Context vectors have dimension $d_\xi=d_p$.
\end{assumption} 
\begin{assumption} \label{ass:basis_dyn}
    The system parameters $p$ of all dynamics $f$ in a basis $\mathcal{B}$ of $\F$ are known.
\end{assumption} 
\begin{proposition}[Identification under linearity. Proof in \Cref{app:params_estimation}] 
    Under \Cref{ass:linear_dyn,ass:linear_model,ass:unicity_param,ass:dim_code,ass:basis_dyn}, system parameters are perfectly identified on new environments if the dynamics model $g$ and hypernetwork $A$ satisfy $\forall f \in \mathcal{B}$ with system parameter $p$, $g_{A(p)}=f$.
    \label{prop:context_bijectivity}
\end{proposition}
Intuitively, \Cref{prop:context_bijectivity} says that given some observations representative of the degrees of variation of the data (a basis of $\F$) and given the system parameters for these observations (\Cref{ass:basis_dyn}), we are guaranteed to recover the parameters of new environments for a family systems. 
This strong guarantee requires strong conditions.
\Cref{ass:linear_dyn,ass:linear_model} state that the systems should be linear \wrt inputs and that the dynamics model should be linear too.
Linearity of the hypernetwork is not an issue as detailed in \Cref{subsec:model}.
\Cref{ass:unicity_param} applies to several real-world systems used in our experiments (\cf \Cref{app:params_estimation} \cref{lemma:injectivity_lv,lemma:injectivity_ns}).
\Cref{ass:dim_code} is not restrictive as we showed that $d_p$ is recovered through cross-validation (\Cref{fig:ablation_dim}).

We propose an extension of \Cref{prop:context_bijectivity} in \Cref{prop:context_bijectivity_NN} to nonlinear systems \wrt inputs and nonlinear dynamics model $g$.
This alleviates the linearity assumption in \Cref{ass:linear_dyn,ass:linear_model} and better fits our experimental setting.
\begin{proposition}[Local identification under non-linearity. Proof in \Cref{app:params_estimation}] \label{prop:context_bijectivity_NN}
    For linearly parametrized systems, nonlinear \wrt inputs and nonlinear dynamics model $g_\theta$ with parameters output by a linear hypernetwork $A$, $\exists \alpha > 0$ \sut system parameters are perfectly identified $\forall e\in\E$ where $\|\xi^e\|\leq \alpha$ if $\forall f\in \mathcal{B}$ with parameter $p$, $g_{A(\alpha \frac{p}{\|p\|})} = f$.
\end{proposition}
\Cref{prop:context_bijectivity_NN} states that system parameters are recovered for environments with context vectors of small norm, under a rescaling condition on true system parameters.
\Cref{prop:context_bijectivity_NN} explains why estimation error increases when system parameters differ greatly from training ones, as these systems are more likely to violate the norm condition.

\section{Related Work}
\label{sec:related_work}
We review Out-of-Distribution (OoD), Multi-Task Learning (MTL) and meta-learning methods and their existing extensions to dynamical systems.

\paragraph{Learning in Multiple Environments}
OoD methods extend the ERM objective to learn domain invariants \eg via robust optimization \citep{Sagawa2020} or Invariant Risk Minimization (IRM) \cite{Arjovsky2019, KruegerCHZBZLC2021}.
However, they are not adapted to our problem as an unique model is learned. 
\ours is closer to meta-learning and MTL.
A standard meta-learning approach is gradient-based meta-learning (GBML), which learns a model initialisation through bi-level optimization.
GBML can then adapt to a new task with few gradient steps. 
The standard GBML method is MAML \citep{Finn2017}, extended in various work.
ANIL \citep{Raghu2020} restricts meta-learning to the last layer of a classifier while other work improve adaptation by preconditioning the gradient \citep{Lee2018,Flennerhag2020,Park2019} \eg Meta-SGD \citep{Li2017} learns dimension-wise inner-loop learning rates.
Contextual meta-learning approaches in \citet{ZintgrafSKHW2019,Garnelo2018} partition parameters into context parameters, adapted on each task, and meta-trained parameters, shared across tasks. 
\ours follows the same objective of learning a low-dimensional representation of each task but generalizes these approaches with hypernetworks.
For MTL, a standard approach is hard-parameter sharing which shares earlier layers of the network \citep{Caruana1997}.
Several extensions were proposed to learn more efficiently from a set of related tasks \citep{Rebuffi2017, Rebuffi2018}. 
Yet, MTL does not address adaptation to new tasks, which is the focus of \ours. 
Some extensions have also considered this problem, mainly for classification \citep{Wang2021MTL,Requeima2019}.

\paragraph{Generalization for Dynamical Systems}
Only few work have considered generalization for dynamical systems.
LEADS \citep{Yin2021LEADS} is a MTL approach that performs adaptation in functional space.
\ours operates in parameter space, making adaptation more expressive and efficient, and scales better with the number of environments as it does not require training a full new network per environment as LEADS does.
A second work is DyAd \citep{Wang2021c}, a context-aware meta-learning method.
DyAd adapts the dynamics model by decoding a time-invariant context, obtained by encoding observed states.
However, unlike \ours, DyAd uses weak supervision obtained from  physics quantities to supervise the encoder, which may not always be possible.
Moreover, it performs AdaIN modulation (instance normalization + FiLM), a particular case of hypernetwork decoding, which performed worse than \ours in our experiments.

\section{Conclusion}
We introduced \ours, a new framework to learn context-informed data-driven dynamics models on multiple environments.
\ours generalizes with little retraining and few data to new related physical systems and outperforms prior methods on several real-world nonlinear dynamics.
Many promising applications of \ours are possible, notably for spatiotemporal problems, \eg partially observed systems,
reinforcement learning, or NN-based simulation.

\paragraph{Acknowledgement}
We acknowledge the financial support from DL4CLIM ANR-19-CHIA-0018-01, DEEPNUM ANR-21-CE23-0017-02, OATMIL ANR-17-CE23-0012, RAIMO ANR-20-CHIA-0021-01 and LEAUDS ANR-18-CE23-0020.

\bibliographystyle{icml2022}
\bibliography{ref}

\newpage

\appendix

\twocolumn[
\icmltitlerunning{Generalizing to New Physical Systems via Context-Informed Dynamics Model: Supplementary Material}
\icmltitle{{Generalizing to New Physical Systems via Context-Informed Dynamics Model} \\
\normalfont{Supplementary Material}}
\vskip 0.1in
]

\section{Discussion}
\label{app:discussion}
We discuss in more details the originality and differences of \ours \wrt several Multi-Task Learning (MTL) and gradient-based or contextual meta-learning methods illustrated in \Cref{fig:baselines}.
We consider CAVIA \citep{ZintgrafSKHW2019}, MAML \citep{Finn2017}, ANIL \citep{Raghu2020}, hard-parameter sharing MTL \citep{Caruana1997,Ruder2017}, LEADS \citep{Yin2021LEADS}.

\subsection{Adaptation Rule}
\label{app:adaptation_rule}
We compare the adaptation rule in \Cref{eq:adaptation_rule} \wrt these work.

\paragraph{GBML}
Given $k$ gradient steps, MAML defines 
\begin{gather}
   \theta^e = \theta^c + (- \eta \sum_{i=0}^k \nabla_{\theta} \Loss(\theta^e_i, \D^e))  \label{eq:maml} \\ \text{where~} 
   \begin{cases}
       \theta^e_{i+1}=\theta^e_{i} - \eta \nabla_{\theta} \Loss(\theta^e_i, \D^e) & i>0\\
       \theta^e_0 = \theta^c & i=0 \notag \\ 
   \end{cases}
\end{gather}
With $\delta\theta^e\triangleq-\eta\sum_{i=0}^k \nabla_{\theta} \Loss(\theta^e_i, \D^e)$, \Cref{eq:adaptation_rule} thus includes MAML. 
ANIL and related GBML methods \citep{Lee2019,Bertinetto2018} restrict \Cref{eq:maml} to parameters of the final layer, while remaining parameters are shared.
\begin{figure}[t!]
    \centering
    \includegraphics[width=\linewidth]{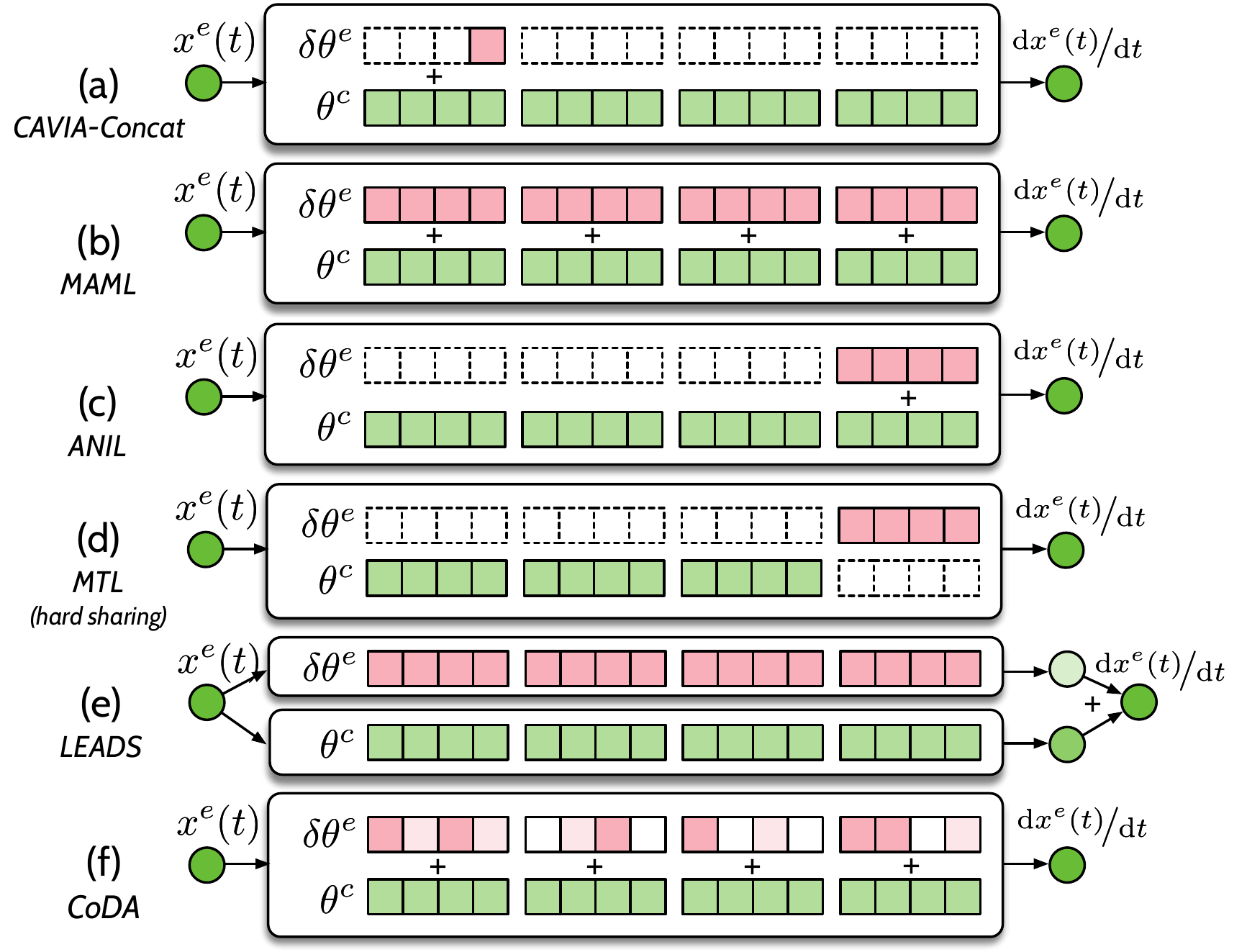}
	\caption{Illustration of representative baselines for multi-environment learning. 
	Shared parameters are {\color{blue} blue}, environment-specific parameters are {\color{red} red}. 
	\mytag{(a)}{approach_a} CAVIA-Concat acts upon the bias of the first layer with conditioning via concatenation. 
	\mytag{(b)}{approach_b} MAML acts upon all parameters without penalization nor prior structure information. 
	\mytag{(c)}{approach_c} ANIL restricts meta-learning to the final layer. 
	\mytag{(d)}{approach_d} Hard-sharing MTL train the final layer from scratch, while the remaining is a \textit{hard-shared}.
	\mytag{(e)}{approach_e} LEADS sums the output of a common and a environment-specific network. 
	\mytag{(f)}{approach_f} \ours acts upon a subspace of the parameter space with a locality constraint.} \label{fig:baselines}
	\vskip -0.05in
\end{figure}

\paragraph{MTL}
MTL models can be identified to \Cref{eq:maml}.
They fix $\theta^c\triangleq \vzero$, removing the ability of performing fast adaptation as parameters are retrained from scratch instead of being initialized to $\theta^c$.
Hard-parameter sharing MTL restricts the sum in \Cref{eq:maml} to the final layer, as ANIL. LEADS sums the outputs of a shared and an environment specific network, thus splits parameters into two independent blocks that do not share connections.

\subsection{Decoding for Context-Informed Adaptation}
\label{app:conditioning_contextual}
We show that conditioning strategies in contextual meta-learning for decoding context vectors $\xi^e$ into $\delta\theta^e$ are a special case of hypernetwork-decoding.
The two main approaches are conditioning via concatenation and conditioning via feature modulation \aka FiLM \citep{Perez2018}.

\subsubsection{Conditioning via Concatenation}
We show that conditioning via concatenation is equivalent to a linear hypernetwork $A_\phi:\xi^e\mapsto W\xi^e + \theta^c$ with $\phi=\{\theta^c, W\}$ that only predicts the bias of the first layer of $g_\theta$.

We assume that $g_\theta$ has $N$ layers and analyze the output of the first layer of $g_\theta$, omitting the nonlinearity, when the input $x \in \mathbb{R}^{d_x}$ in an environment $e\in\E$ is concatenated to a context vector $\xi^e \in \mathbb{R}^{d_\xi}$.
We denote $x \| \xi^e$ the concatenated vector, $n_h$ the number of hidden units of the first layer, $W^1 \in \mathbb{R}^{n_h \times(d_x+d_\xi)}$ and $b^1 \in \mathbb{R}^{n_h}$ the weight matrix and bias term of the first layer, $W^2, \cdots, W^N$ and $b^2, \cdots, b^N$ those of the following layers.
The output of the first layer is
$$
    y^1=W^1 \cdot x \| \xi^e+b^1
$$
We split $W^1$ along rows into two weight matrices, $W_x^1 \in \mathbb{R}^{n_h \times d_x}$ and $W_\xi^1 \in \mathbb{R}^{n_h \times d_\xi}$ \sut
$$
    y^1=W_x^1 \cdot x+W_\xi^1 \cdot \xi^e+b^1
$$
$b_\xi^1 \triangleq W_\xi^1 \cdot \xi^e+b^1$ does not depend on $x$ and corresponds to an environment-specific bias. Thus, concatenation is included in \Cref{eq:adaptation_rule} when
\begin{align*}
    \theta^c &\triangleq \{W_x^1, b^1, W^2, b^2, \cdots, W^N, b^N\} \\
    \delta\theta^e &\triangleq \{~~0~~,b_\xi^1, ~~0~~, 0~, \cdots, ~~~0~~~, 0\}
\end{align*}
where $\delta\theta^e$ is decoded via a hypernetwork with parameters $\{\theta^c, W\triangleq(0, W_\xi^1, 0, \cdots, 0)\}$.

\subsubsection{Conditioning via Feature Modulation}
We show that conditioning via FiLM is equivalent to a linear hypernetwork $A_\phi:\xi^e\mapsto W\xi^e + \theta^c$ with $\phi=\{\theta^c, W\}$ that only predicts the batch norm (BN) statistics of $g_\theta$.

For simplicity, we focus on a single BN layer and denote $\{h_i\}_{i=1}^M$, $M$ feature maps output by preceding convolutional layers.
These feature maps are first normalized then rescaled with an affine transformation.
Rescaling is similar to a FiLM layer that transforms linearly $\{h_i\}_{i=1}^M$ with:
$$
    \forall i\in\{1,\cdots,M\}, \text{FiLM}(h_i)=\gamma_i \odot h_i + \beta
$$ 
where $\gamma,\beta\in\mathbb{R}^M$ are output by a NN $f_\psi$ conditioned on the context vectors $\xi^e$ \ie $[\gamma,\beta]=f_\psi(\xi^e)$.
In general, $f_\psi$ is linear \sut $f_\psi(\xi^e) \triangleq W_\xi\xi^e+b_\xi$, with $\psi=\{W_\xi, b_\xi\}$.
Then $\gamma=W_\xi^\gamma \xi^e + b_\xi^\gamma, \beta=W_\xi^\beta \xi^e + b_\xi^\beta$.

Thus, for this layer, modulation is included in \Cref{eq:adaptation_rule} when
\begin{align*}
    \delta\theta^e &\triangleq W\xi^e = \{W_\xi^\gamma \xi^e, W_\xi^\beta \xi^e\} \\
    \theta^c &\triangleq b_\xi = \{b_\xi^\gamma,b_\xi^\beta\} 
\end{align*}
where $\delta\theta^e$ is decoded via hypernetwork $f_\psi \triangleq A_\phi$ with parameters $\phi=\{\theta^c\triangleq b_\xi, W\triangleq W_\xi\}$.

\section{Proofs}
\label{app:proof}

\begin{customprop}{\ref{prop:lr_linearit}}
    Given a class of linearly parametrized dynamics $\F$ with $d_p$ varying parameters, $\forall \theta^c\in\mathbb{R}^{d_\theta}$, subspace $\G_{\theta^c}$ in \Cref{def:grad_dir} is low-dimensional and satisfies $\dim(\G_{\theta^c}) \leq d_p \ll d_{\theta}$.
\end{customprop} 

\begin{proof}
    We define the linear mapping $\psi: p \in \mathbb{R}^{d_p} \rightarrow f \in \F$ from parameters to dynamics \sut $\psi(\mathbb{R}^{d_p})=\F$. 
    Given this linear mapping, we first prove the following lemma: $\dim(\F)\leq d_p$.
    The proof is based on surjectivity of $\psi$ onto $\F$, given by definition. 
    We define $\{b_i\}_{i=1}^{d_p}$ a basis of $\mathbb{R}^{d_p}$.
    Given $f \in \F$, $\exists p \in \mathbb{R}^{d_p}, \psi(p)=f$. 
    We note $p = \sum_{i=1}^{d_p} \lambda_i b_i$ where $\forall i, \lambda_i \in \mathbb{R}$. 
    Then $\psi(p) = \sum_{i=1}^{d_p} \lambda_i \psi(b_i)$. 
    We extract a basis from $\{\psi(b_i)\}_{i=1}^{d_p}$ and denote $d_f \leq d_p$ the number of elements in this basis. 
    This basis forms a basis of $\F$ \ie $d_f=\dim(\F) \leq d_p$.
    
    Now, given $\theta \in\mathbb{R}^{d_\theta}$ and $f^e \in \F$. We precise that given a (probability) measure $\rho_\gX$ on $\gX\subset\R^d$, the function space $\F\subset L^2(\rho_x,\R^d)$, then
    \begin{equation*}
        \Loss(\theta, \D^e) \triangleq \int_{\gX} \|(f^e-g_\theta)(x)\|_2^2\diff \rho_\gX(x) = \|f^e - g_\theta\|^2_2
    \end{equation*}
    The gradient of $\Loss(\theta, \D^e)$ is then
    \begin{align*}
        \nabla_\theta \Loss({\theta^c}, &\D^e) = \nabla_\theta \int_{\gX} \|f^e(x)-g_{\theta^c}(x)\|_2^2\diff \rho_\gX(x) \\
        &= \int_{\gX} \nabla_\theta  \|f^e(x)-g_{\theta^c}(x)\|_2^2\diff \rho_\gX(x) \\
        &= -2\int_{\gX} \mathbf{J}_\theta g_{\theta^c}(x)^\top(f^e(x)-g_{\theta^c}(x))\diff \rho_\gX(x) \\
        &= -2 D_\theta g_{\theta^c}^\top (f^e - g_{\theta^c})
    \end{align*}
    where $\mathbf{J}_\theta g_\theta(x)$ is the Jacobian matrix of $g_{\theta^c}$ \wrt $\theta$ at point $x$. $\theta\mapsto D_\theta g_{\theta^c}$ is the differential of $g_\theta$. Note that $D_\theta g_{\theta^c}:\R^{d_\theta} \to \F$ is a linear map (analogue of Jacobian matrix). $D_\theta g_{\theta^c}^\top: \F^\star \rightarrow \mathbb{R}^{d_\theta}$ denotes its adjoint (analogue of transposed matrix), which is also a linear map. 
    
    As $\G_{\theta^c}\subseteq Im(D_\theta g_\theta^\top)$, then according to Rank-nullity theorem, $\dim(\G_{\theta^c}){}\leq \dim(Im(D_\theta g_{\theta^c}^\top)){}= \dim(\F) - \dim(Ker(D_\theta g_{\theta^c}^\top)){}\leq \dim(\F){}\leq d_p$.
\end{proof}

\begin{customprop}{\ref{prop:optimal_context}}
    Given $\{\theta^c,W\}$ fixed, if $\|\cdot\|{}=\ell_2$, then \Cref{eq:coda-ts} is quadratic. 
   If $\lambda^\prime W^\top W$ or $\bar{H}^e(\theta^c)=W^\top \nabla^2_\theta \Loss(\theta^c, \D^e) W$ are invertible then $\bar{H}^e(\theta^c)+\lambda^\prime W^\top W$ is invertible except for a finite number of $\lambda^\prime$ values.
   The problem in \Cref{eq:coda-ts} is then also convex and admits an unique solution, $\{{\xi^e}^\star\}_{e\in\Ead}$.
   With $\lambda^\prime \triangleq 2 \lambda$,
    \begin{equation*}
        {\xi^e}^* = -\Big({\bar{H}^e(\theta^c)+\lambda^\prime W^\top W}\Big)^{-1} W^\top \nabla_\theta\Loss(\theta^c, \D^e)
    \end{equation*}
    $\bar{H}^e(\theta^c)+\lambda^\prime W^\top W$ is invertible $\forall \lambda^\prime$ except a finite number of values if $\bar{H}^e(\theta^c)$ or $\lambda^\prime W^\top W$ is invertible. 
\end{customprop}
\begin{proof}
    When $\|\cdot\|=\ell_2$, we consider the following second order Taylor expansion of $\Loss_{\mathrm{r}}(\theta, \D^e)\triangleq\Loss(\theta, \D^e) + \lambda \|\theta-\theta^c\|^2_2$ at $\theta^c$, where $\delta\theta^e=\theta-\theta^c=W\xi^e$.
    \begin{multline}
        \smash[t]{\Loss_{\mathrm{r}}(\theta^c + \delta\theta^e, \D^e) = \Loss(\theta^c, \D^e) + \nabla_\theta \Loss(\theta^c, \D^e)^\top \delta\theta^e} + \\
        \frac{1}{2}{\delta\theta^e}^\top \Big(\nabla^2_\theta \Loss(\theta^c, \D^e)+ 2\lambda \mathrm{Id}\Big) \delta\theta^e + o(\|\delta\theta^e\|^3_2)
        \label{eq:taylor_expansion}
    \end{multline}
    With $\delta\theta^e=W\xi^e$, we expand \Cref{eq:taylor_expansion} into
    \begin{align*}
        &\Loss_{\mathrm{r}}(\theta^c + W\xi^e, \D^e) = \Loss(\theta^c, \D^e) + (W^\top \nabla_\theta \Loss(\theta^c, \D^e))^\top \xi^e \\
        &+ \frac{1}{2}{\xi^e}^\top (W^\top \nabla^2_\theta \Loss(\theta^c, \D^e) W + 2\lambda W^\top W) \xi^e + o(\|\delta\theta^e\|^3_2)
    \end{align*}
    \ie with $\bar{H}^e(\theta^c)=W^\top \nabla^2_\theta \Loss(\theta^c, \D^e) W$ and $\lambda^\prime=2\lambda$
    \begin{align}
        &\Loss_{\mathrm{r}}(\theta^c + W\xi^e, \D^e) = \Loss(\theta^c, \D^e) + (W^\top \nabla_\theta \Loss(\theta^c, \D^e))^\top \xi^e \notag\\
        &+ \frac{1}{2}{\xi^e}^\top \Big(\bar{H}^e(\theta^c) + \lambda^\prime W^\top W\Big) \xi^e + o(\|\delta\theta^e\|^3_2)
        \label{eq:taylor_exp_low}
    \end{align}
    \Cref{eq:taylor_exp_low} is quadratic.
    If $\bar{H}^e(\theta^c)+\lambda^\prime W^\top W$ is invertible, then the problem is also convex with unique solution
    \begin{equation*}
        {\xi^e}^* = -\Big({\bar{H}^e(\theta^c)+\lambda^\prime W^\top W}\Big)^{-1} W^\top \nabla_\theta\Loss(\theta^c, \D^e)
    \end{equation*}
    
    $\bar{H}^e(\theta^c)$ and $\lambda^\prime W^\top W$ are two square matrices. 
    The application $p:\lambda^\prime\mapsto \det(\bar{H}^e(\theta^c)+\lambda^\prime W^\top W)$ is well-defined and forms a continuous polynomial.
    Thus either it equals zero or it has a finite number of roots.
    If $\bar{H}^e(\theta^c)$ or $\lambda^\prime W^\top W$ is invertible, then $p(0)=\det(\bar{H}^e(\theta^c))\neq0$ or $p(\infty)\sim \det(\lambda^\prime W^\top W)\neq0$.
    Thus $p\neq0$ has a finite number of roots \ie $\bar{H}^e(\theta^c)+\lambda^\prime W^\top W$ is invertible $\forall \lambda^\prime$ except a finite number of values corresponding to the roots of $p$.
\end{proof}

\section{System Parameter Estimation}
\label{app:params_estimation}
\begin{customprop}{\ref{prop:context_bijectivity}}
    Under \Cref{ass:linear_dyn,ass:linear_model,ass:unicity_param,ass:dim_code,ass:basis_dyn}, system parameters are perfectly identified on new environments if the dynamics model $g$ and hypernetwork $A$ satisfy $\forall f \in \mathcal{B}$ with system parameter $p$, $g_{A(p)}=f$.
\end{customprop}
\begin{proof}
    We define the linear mapping $\psi: p \in \mathbb{R}^{d_p} \rightarrow f \in \F$ from parameters to dynamics \sut $\psi(\mathbb{R}^{d_p})=\F$ (\Cref{ass:linear_dyn}). 
    Unicity of parameters (\Cref{ass:unicity_param}) implies that $\psi$ is bijective with inverse $\psi^{-1}$, thus $\dim(\F)=\dim(\mathbb{R}^{d_p})=d_p$.
    Given a basis $\mathcal{B}=\{f_i\}_{i=1}^{d_p}$ of $\F$, we denote $p_i=\psi^{-1}(f_i)$.
    We fix $g, A$ \sut $\forall i \in \{1, ..., d_p\}, g_{A(p_i)} = f_i = \psi(p_i)$.
    This is possible as $f_i$ and $g$ are linear \wrt inputs (\Cref{ass:linear_dyn,ass:linear_model}) and $p_i$ are known (\Cref{ass:basis_dyn}).
    
    $g, A$ are linear (\Cref{ass:linear_model}), thus $g_{A(\cdot)}$ is linear with inputs in $\mathbb{R}^{d_\xi}$.
    Then, $\dim(\Image(g_{A(\cdot)})) \leq d_\xi$.
    Moreover, $\forall i \in \{1, ..., d_p\}, f_i \in \Image(g_{A(\cdot)})$, thus $\F \subset \Image(g_{A(\cdot)})$ \ie $d_p \leq \dim(\Image(g_{A(\cdot)}))$.
    Thus, $d_p \leq \dim(\Image(g_{A(\cdot)})) \leq d_\xi$.
    \Cref{ass:dim_code} states that $d_\xi=d_p$, s.t. $\dim(\Image(g_{A(\cdot)}))=d_p$.
    As $\F \subset \Image(g_{A(\cdot)})$ and $\dim(\F)=\dim(\Image(g_{A(\cdot)}))$, $\F=\Image(g_{A(\cdot)})$ \ie $g_{A(\cdot)}$ is surjective onto $\F$. 
    As $\dim(\F)=d_\xi$, the dimension of the input space, $g_{A(\cdot)}$ is bijective.
    
    By bijectivity of $\psi$, $\{p_i\}_{i=1}^{d_p}$ forms a basis of $\mathbb{R}^{d_p}$. 
    $g_{A(\cdot)}$ and $\psi$ map this basis to the same basis $\{f_i\}_{i=1}^{d_p}$ of $\F$. 
    As both mappings are bijective, this implies that $g_{A(\cdot)}=\psi(\cdot)$.
    This means that $\forall e \in \E, {g_{A}}^{-1}(f^e)=\psi^{-1}(f^e)$ \ie system parameters $p^e$ are recovered.
\end{proof}

\begin{customprop}{\ref{prop:context_bijectivity_NN}}
    For linearly parametrized systems, nonlinear \wrt inputs and nonlinear dynamics model $g_\theta$ with parameters output by a linear hypernetwork $A$, $\exists \alpha > 0$ \sut system parameters are perfectly identified $\forall e\in\E$ where $\|\xi^e\|\leq \alpha$ if $\forall f\in \mathcal{B}$ with parameter $p$, $g_{A(\alpha \frac{p}{\|p\|})} = f$.
\end{customprop}
\begin{proof}
    On environment $e\in\E$, $g_{\theta^e}$ is differentiable \wrt $\theta^e=A(\xi^e)=\theta^c+W\xi^e \in \mathbb{R}^{d_\theta}$.
    We perform a first order Taylor expansion of $g_{A(\cdot)}$ around $\vzero$.
    We note $\alpha > 0,$ \sut $\forall \xi^e\in\mathbb{R}^{d_\xi}$ that satisfy $\|\xi^e\|<\alpha$, we have $g_{\theta^e} = g_{\theta^c} + \nabla_\theta g_{\theta^c} W \xi^e$.
    $g_{A(\cdot)}$ is then linear in the neighborhood of $\vzero$ defined by $\alpha$.
    $\forall i\in\llbracket1, d_p\rrbracket, \alpha \frac{p_i}{\|p_i\|}$ belongs to this neighborhood \sut the proof of \Cref{prop:context_bijectivity} applies to this neighborhood if $\forall i\in\llbracket1, d_p\rrbracket, g_{A(\alpha \frac{p_i}{\|p_i\|})} = f_i$, where $\mathcal{B}=\{f_i\}_{i=1}^{d_p}$ is a basis of $\F$.
\end{proof}

We now show the validity of the unicity condition (\Cref{ass:unicity_param}) for two linearly parametrized systems.
\begin{lemma}
    There is an unique set of parameters in $\mathbb{R}^{4}$ for a Lotka-Volterra (\texttt{LV}) system.
    \label{lemma:injectivity_lv}
\end{lemma}
\vspace{-1em}
\begin{proof}
    With $\psi: c \triangleq (\alpha, \beta, \delta, \gamma) \mapsto [(\begin{array}{l} x\\ y \end{array})  \mapsto (\begin{array}{l} {\alpha} {x} - {\beta} {x y}\\ {\delta} {x y} - {\gamma} {y} \end{array})]$
    a surjective linear mapping from $\mathbb{R}^4$ to $\F$ (all \texttt{LV} systems are parametrized).
    Injectivity of $\psi$ \ie $\psi(c_1)=\psi(c_2) \iff c_1=c_2$ will imply bijectivity \ie unicity of parameters for a \texttt{LV} system. 
    As $\psi$ is linear, injectivity is equivalent to $\psi(c)=0 \iff c=0$, shown below:
    \begin{align*}
        \psi(c)=0 &\iff
        \forall \Big(\begin{array}{l} x\\ y \end{array}\Big), \Big(\begin{array}{l} {x} ({\alpha} - {\beta} y) \\
        ({\delta} {x} - {\gamma}) {y} \end{array}\Big)  = \Big(\begin{array}{l} 0\\ 0 \end{array}\Big) \\
        &\iff \forall \Big(\begin{array}{l} x\\ y \end{array}\Big), \Big(\begin{array}{l} {\alpha} - {\beta} y \\
        {\delta} {x} - {\gamma} \end{array}\Big)  = \Big(\begin{array}{l} 0\\ 0 \end{array}\Big) \\
        &\iff c = (\alpha, \beta, \delta, \gamma) = (0,0,0,0)
    \end{align*}
\end{proof}

\begin{lemma}
    There is an unique set of parameters in $\mathbb{R}^{d+1}$, where $d$ is the grid size, for a Navier-Stokes (\texttt{NS}) system.
    \label{lemma:injectivity_ns}
\end{lemma}
\begin{proof}
    With $\psi: c \triangleq (\nu, f) \mapsto \Big[w  \mapsto -v \nabla w + \nu \Delta w + f\Big]$, a surjective linear mapping from $\mathbb{R}^{d+1}$ to $\F$ (all \texttt{NS} systems are parametrized), bijectivity of $\psi$ is induced by injectivity \ie $\psi(c_1)=\psi(c_2) \iff c_1=c_2$, shown below:
    \begin{align*}
        &\psi(c_1)=\psi(c_2) \\
        \iff& \forall w,  -v \nabla w + \nu_1 \Delta w + f_1 =  -v \nabla w + \nu_2 \Delta w + f_2\\
        \iff& \forall w,  (\nu_1 - \nu_2) \Delta w =  -(f_1-f_2)\\
        \iff& (\nu_1, f_1) = (\nu_2, f_2) \iff c_1=c_2
    \end{align*}
\end{proof}

\section{Low-Rank Assumption}
\label{app:low-rank-nonlinear}
When the systems are nonlinearly parametrized, we show empirically with \Cref{fig:eigenvalues} that the low-rank assumption is still reasonable for two different systems.
\begin{figure*}[t]
    \centering
    \subfloat[\texttt{GO}: $k_1$ and $K_1$ vary across $\E$. \label{fig:eigenvalues_go}]{\includegraphics[width=0.66\linewidth]{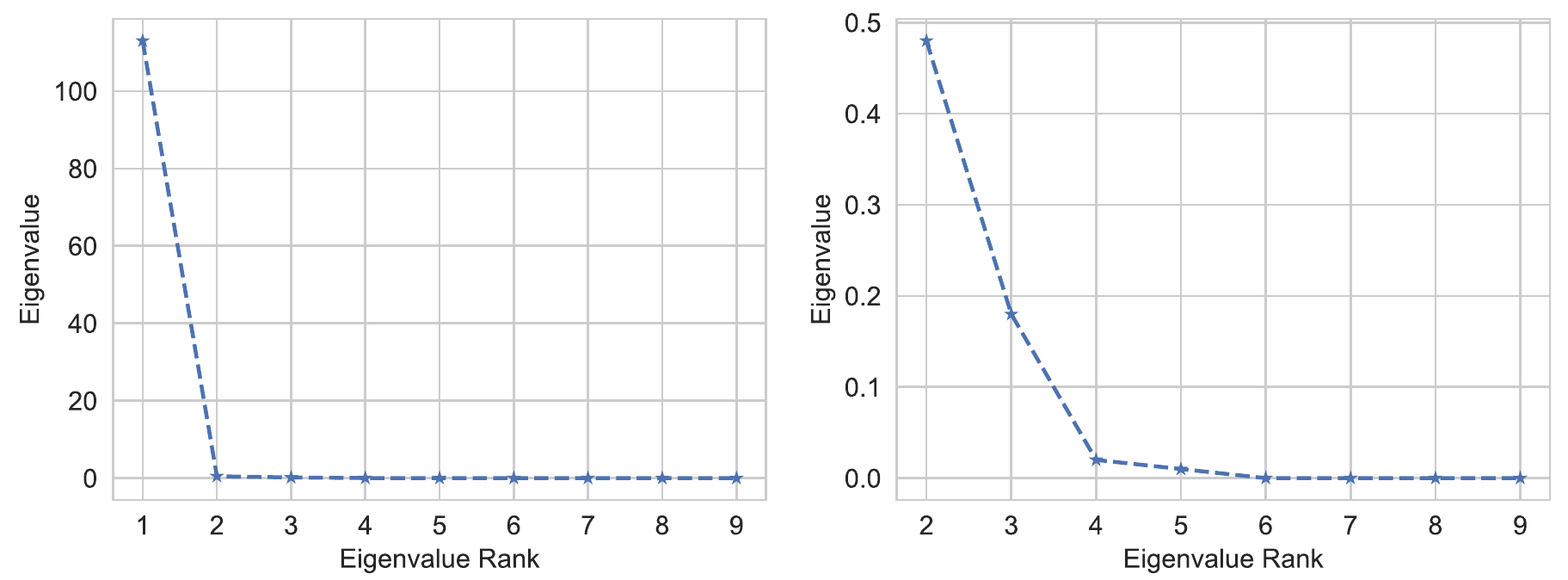}}
    \subfloat[\texttt{Sin}. \label{fig:eigenvalues_sin}]{\includegraphics[width=0.33\linewidth]{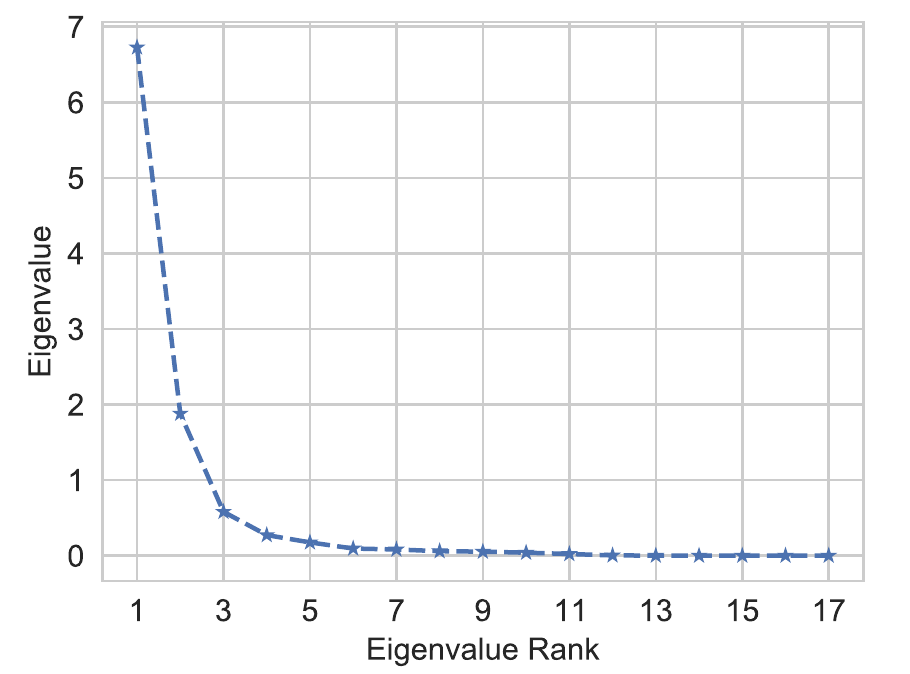}}
    \caption{Ranked singular values of the gradients across environments $\Etr$, $\G_{\theta^c}$ for \ourslone.}
    \label{fig:eigenvalues}
    \vskip -0.2in
\end{figure*}
\paragraph{Glycolitic-Oscillator (\texttt{GO})}
We consider the Glycolitic-Oscillator system (\texttt{GO}), described in \Cref{app:dynsys}, which is nonlinear \wrt $K_1$.
We vary parameters $k_1, K_1$ in \Cref{eq:glycolitic} across environments.
We observe in \Cref{fig:eigenvalues_go} that there are three main gradient directions with SVD. 
The first is the most significant one while the second and third ones are orders of magnitude smaller.

\paragraph{Sinusoidal (\texttt{Sin})}
We consider a sinusoidal family of functions $S(n)= \{f:\mathbb{R}\rightarrow \mathbb{R}| f(x)= \sum_{i=1}^N \lambda_i \sin(\omega_i x + \phi_i)\}$ (\texttt{Sin}).
We sample 20 environments that correspond each to different amplitudes (uniformly sampled in $[0,1]$), frequencies (uniformly sampled in $[0,10]$)) and phases (uniformly sampled in $[0, 3.14]$).
We depict in \Cref{fig:eigenvalues_sin} the evaluation of the singular values at initialization.
\Cref{fig:eigenvalues_sin} shows that the number of directions to consider for convergence is small and that a single direction accounts for a significant amount of the variance in the gradients.
This corroborates the low-rank assumption.

\section{Locality Constraint}
\label{app:locality_upper_bound}
We derive the upper-bounds to $\|\cdot\|$ for two variations.

$\|\cdot\|{}=\ell_2$: we apply triangle inequality to obtain $\Omega=\ell_2^2$
$$
    \|W\xi^e\|_2^2{}\leq\|W\|_2^2\|\xi^e\|_2^2
$$

$\|\cdot\|{}=\ell_1$: we apply Cauchy-Schwartz inequality to obtain $\Omega(W){}=\ell_{1,2}(W)\triangleq\sum_{i=1}^{d_\theta}\|W_{i,:}\|_2$
$$
\|W\xi^e\|_1{}=\sum_{i=1}^{d_\theta}|W_{i,:}\xi^e|{}\leq\|\xi^e\|_2\sum_{i=1}^{d_\theta}\|W_{i,:}\|_2
$$
\Cref{eq:reg} minimizes the $\log$ of the above upper-bounds.

\section{Experimental Settings}
\label{app:experimental_setting}
We present in \Cref{app:dynsys} the equations and the data generation specificities for all considered dynamical systems.

\subsection{Dynamical Systems}
\label{app:dynsys}
\paragraph{Lotka-Volterra (\texttt{LV}, \citealp{Lotka1925})} The system describes the interaction between a prey-predator pair in an ecosystem, formalized into the following ODE:
\begin{equation}
    \begin{aligned}
        \dfrac{\diff x}{\diff t} &= {\alpha} {x} - {\beta} {x y} \\
        \dfrac{\diff y}{\diff t} &= {\delta} {x y} - {\gamma} {y}
    \end{aligned}
    \label{eq:lv}
\end{equation}
where $x, y$ are respectively the quantity of the prey and the predator, $\alpha, \beta, \delta, \gamma$ define how two species interact. 

We generate trajectories on a temporal grid with $\Delta t = 0.5$ and temporal horizon $T=10$. 
We sample on each training environment $N_{\text{tr}}=4$ initial conditions for training from a uniform distribution $p(X_0) = \operatorname{Unif}([1, 3]^2)$. 
We sample for evaluation $32$ initial conditions from $p(X_0)$.
Across environments, $\alpha=0.5, \gamma=0.5$.
For training, we consider $\#\Etr=9$ environments with parameters $\beta, \delta \in \{0.5, 0.75, 1.0\}^2$.
For adaptation, we consider $\#\Ead=4$ environments with parameters $\beta, \delta \in \{0.625, 1.125\}^2$.

\paragraph{Glycolytic-Oscillator (\texttt{GO}, \citealp{Daniels2015})} 
\texttt{GO} describes yeast glycolysis dynamics with the ODE:
\begin{equation}
    \begin{aligned}
        \dfrac{\diff S_1}{\diff t} &= J_0 - \dfrac{k_1 S_1 S_6}{1 + (1/K_1^q) S_6^q} \\
        \dfrac{\diff S_2}{\diff t} &= 2 \dfrac{k_1 S_1 S_6}{1 + (1/K_1^q) S_6^q} - k_2 S_2(N-S_5)-k_6 S_2 S_5 \\
        \dfrac{\diff S_3}{\diff t} &= k_2 S_2 (N-S_5) - k_3 S_3(A-S_6) \\
        \dfrac{\diff S_4}{\diff t} &= k_3 S_3 (A-S_6) - k_4 S_4 S_5 - \kappa (S_4 - S_7) \\
         \dfrac{\diff S_5}{\diff t} &= k_2 S_2 (N-S_5) - k_4 S_4 S_5 - k_6 S_2 S_5 \\
        \dfrac{\diff S_6}{\diff t} &= -2 \dfrac{k_1 S_1 S_6}{1 + (1/K_1^q) S_6^q} + 2 k_3 S_3 (A - S_6) - k_5 S_6 \\
        \dfrac{\diff S_7}{\diff t} &= \psi \kappa (S_4 - S_7) - k S_7
    \end{aligned}
    \label{eq:glycolitic}
\end{equation}
where $S_1, S_2, S_3, S_4, S_5, S_6, S_7$ represent the concentrations of 7 biochemical species.
We generate trajectories on a temporal grid with $\Delta t = 0.05$ and temporal horizon $T=1$. 
We sample on each training environment $N_{\text{tr}}=32$ initial conditions for training from a uniform distribution $p(X_0)$ defined in Table 2 in \citep{Daniels2015}. 
%We sample for adaptation $1$ initial condition from $p(X_0)$.
Across environments, $J_0=2.5, k_2=6, k_3=16, k_4=100, k_5=1.28, k_6=12, q=4, N=1, A=4, \kappa=13, \psi=0.1, k=1.8$.
For training, we consider $\#\Etr=9$ environments with parameters $k_1 \in \{100, 90, 80\}, K_1 \in \{1, 0.75, 0.5\}$.
For adaptation, we consider $\#\Ead=4$ environments with parameters $k_1 \in \{85, 95\}, K_1 \in \{0.625, 0.875\}$.

\paragraph{Gray-Scott (\texttt{GS}, \citealp{Pearson1993})}
The PDE descibes a reaction-diffusion system with complex spatiotemporal patterns through the following 2D PDE:
\begin{equation}
    \begin{aligned}
        \dfrac{\partial u}{\partial t} &= D_u \Delta u -u v^2 + F (1-u) \\
        \dfrac{\partial v}{\partial t} &= D_v \Delta v + u v^2 - (F + k) v
    \end{aligned}
    \label{eq:gray}
\end{equation}
where $u, v$ represent the concentrations of two chemical components in the spatial domain $S$ with periodic boundary conditions. 
$D_u, D_v$ denote the diffusion coefficients respectively for $u, v$ and $F, k$ are the reaction parameters. 

We generate trajectories on a temporal grid with $\Delta t = 40$ and temporal horizon $T=400$. 
$S$ is a 2D space of dimension 32$\times$32 with spatial resolution of $\Delta s=2$.
We define initial conditions $(u_0,v_0)\sim p(X_0)$ by uniformly sampling three two-by-two squares in $S$.
These squares trigger the reactions.
$(u_0,v_0)=(1-\epsilon, \epsilon)$ with $\epsilon=0.05$ inside the squares and $(u_0,v_0)=(0, 1)$ outside the squares.
We sample on each training environment $N_{\text{tr}}=1$ initial conditions for training. 
%We sample for adaptation $1$ initial condition.
Across environments, $D_u = 0.2097, D_v = 0.105$.
For training, we consider $\#\Etr=4$ environments with parameters $F \in \{0.30, 0.39\}, k \in \{0.058, 0.062\}$.
For adaptation, we consider $\#\Ead=4$ environments with parameters $F \in \{0.33, 0.36\}, k \in \{0.59, 0.61\}$.

\paragraph{Navier-Stokes (\texttt{NS}, \citealp{stokes1851effect})} 
\texttt{NS} describes the dynamics of incompressible flows with the 2D PDE:
\begin{equation}
    \begin{aligned}
        \dfrac{\partial w}{\partial t} &= -v \nabla w + \nu \Delta w + f \text{~where~} w = \nabla \times v\\ 
        \nabla v &= 0 \\ 
    \end{aligned}
    \label{eq:navier}
\end{equation}
where $v$ is the velocity field, $w=\nabla \times v$ is the vorticity. 
Both $v, w$ lie in a spatial domain $S$ with periodic boundary conditions, $\nu$ is the viscosity and $f$ is the constant forcing term in the domain $S$. 
We generate trajectories on a temporal grid with $\Delta t = 1$ and temporal horizon $T=10$. 
$S$ is a 2D space of dimension 32$\times$32 with spatial resolution of $\Delta s=1$.
We sample on each training environment $N_{\text{tr}}=16$ initial conditions for training from $p(X_0)$ as in \citet{Li2021}.
Across environments, $f(X,Y)= 0.1 (\sin(2 \pi (X + Y)) + \cos(2 \pi (X + Y)))$.
For training, we consider $\#\Etr=5$ environments with parameters $\nu \in \{8\cdot 10^{-4}, 9\cdot 10^{-4}, 1.0\cdot 10^{-3}, 1.1\cdot 10^{-3}, 1.2\cdot 10^{-3}\}$.
For adaptation, we consider $\#\Ead=4$ environments with parameters $\nu \in \{8.5\cdot 10^{-4}, 9.5\cdot 10^{-4}, 1.05\cdot 10^{-3}, 1.15\cdot 10^{-3}\}$.

\subsection{Implementation and Hyperparameters}
\label{app:implem_hyperparam}

\paragraph{Architecture}
We implement the dynamics model $g_\theta$ with the following architectures: 
\begin{itemize}
    \item \texttt{LV}, \texttt{GO}: 4-layer MLPs with hidden layers of width 64.
    \item \texttt{GS}: 4-layer ConvNet with 64-channel hidden layers, and $3\times 3$ convolution kernels
    \item \texttt{NS}: Fourier Neural Operator \cite{Li2021} with 4 spectral convolution layers. 12 frequency modes and hidden layers with width 10.
\end{itemize}
We apply Swish activation \citep{Ramachandran2018}. 
The hypernet $A$ is a single affine layer NN.

\paragraph{Optimizer}
We use the Adam optimizer \citep{Kingma2015} with learning rate $10^{-3}$ and $\left(\beta_{1}, \beta_{2}\right)=(0.9,0.999)$. 
We apply early stopping.
All experiments are performed with a single NVIDIA Titan Xp GPU on an internal cluster. 
We distribute training by batching together predictions across trajectories to reduce running time.
States across batch elements are concatenated.

\paragraph{Hyperparameters}
We define hyperparameters for the following models:
\begin{enumerate*}[label=(\alph*)]
    \item \ours: 
    \begin{itemize*}
    \item \texttt{LV}: $\lambda_{\xi}=10^{-4}$, $\lambda_{\ell_{1}}=10^{-6}$, $\lambda_{\ell_{2}}=10^{-5}$
    \item \texttt{GO}: $\lambda_{\xi}=10^{-3}$, $\lambda_{\ell_{1}}=10^{-7}$, $\lambda_{\ell_{2}}=10^{-7}$
    \item \texttt{GS}: $\lambda_{\xi}=10^{-2}$, $\lambda_{\ell_{1}}=10^{-5}$, $\lambda_{\ell_{2}}=10^{-5}$
    \item \texttt{NS}: $\lambda_{\xi}=10^{-3}$, $\lambda_{\ell_{1}}=2\cdot 10^{-3}$, $\lambda_{\ell_{2}}= 2\cdot 10^{-3}$
    \end{itemize*}
    \item LEADS: we use the same parameters as \citet{Yin2021LEADS}.
    \item GBML: the outer-loop learning rate is $10^{-3}$, we apply 1-step inner-loop update for training and adaptation to maintain low running times. The inner-loop learning rate for each system is: \begin{itemize*}
    \item \texttt{LV}: $0.1$
    \item \texttt{GO}: $0.01$
    \item \texttt{GS}: $10^{-3}$
    \item \texttt{NS}: $10^{-3}$
    \end{itemize*}. These values are also used to initialize the per-parameter inner-loop learning rate in Meta-SGD.
\end{enumerate*}

\section{Trajectory Prediction Visualization}
\label{app:visualization}
We visualize in \Cref{fig:gs_traj,fig:ns_traj} the prediction MSE by MAML, LEADS, CAVIA-Concat and \ourslone along ground truth trajectories on the PDE systems \texttt{NS} and \texttt{GS}.
We consider a new test trajectory on an \textit{Adaptation} environment $e\in\Ead$ with parameters defined in the caption.

\newcolumntype{P}[1]{>{\centering\arraybackslash}p{#1}}
\newcolumntype{M}[1]{>{\centering\arraybackslash}m{#1}}
\begin{figure*}
    \centering
    \fontfamily{Cabin-TLF}\selectfont
    \setlength{\tabcolsep}{6pt}
    \renewcommand{\arraystretch}{5}
    \tiny
    \begin{tabular}{M{0.8cm}c}
        Ground Truth & \includegraphics[align=c,width=0.8\linewidth]{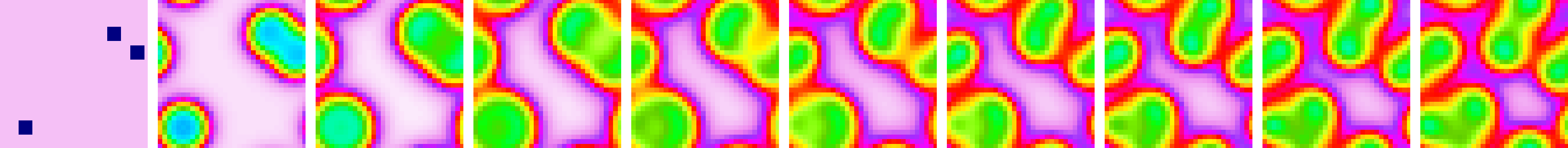} \\
        MAML & \includegraphics[align=c,width=0.8\linewidth]{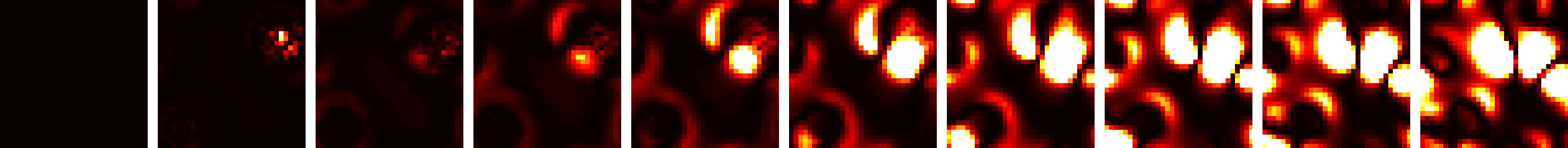} \\
        LEADS & \vspace{-0.5em}\includegraphics[align=c,width=0.8\linewidth]{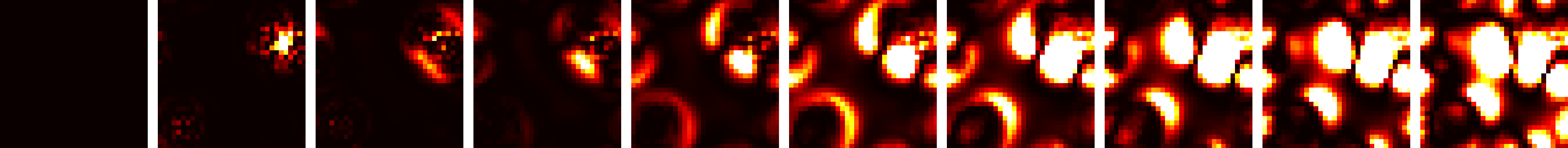} \\
        CAVIA-Concat & \vspace{-0.5em}\includegraphics[align=c,width=0.8\linewidth]{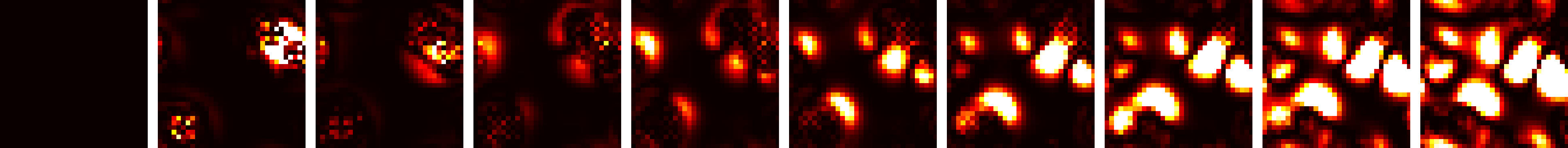} \\
        CoDA (\textit{Ours}) & \includegraphics[align=c,width=0.8\linewidth]{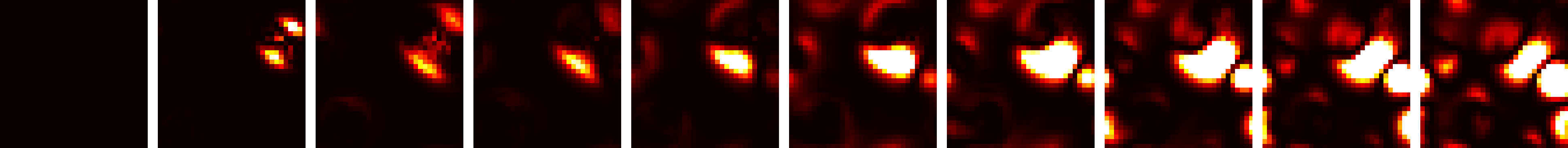}
    \end{tabular}
    \caption{Adaptation to new \texttt{GS} system - $(F,k,D_u,D_v)=(0.033, 0.061, 0.2097, 0.105)$. Ground-truth trajectory and prediction MSE per frame for MAML, LEADS, CAVIA-Concat and CoDA.}
    \label{fig:gs_traj}
\end{figure*}
\begin{figure*}
    \centering
    \fontfamily{Cabin-TLF}\selectfont
    \setlength{\tabcolsep}{6pt}
    \renewcommand{\arraystretch}{5}
    \tiny
    \begin{tabular}{M{0.8cm}c}
        Ground Truth & \includegraphics[align=c,width=0.8\linewidth]{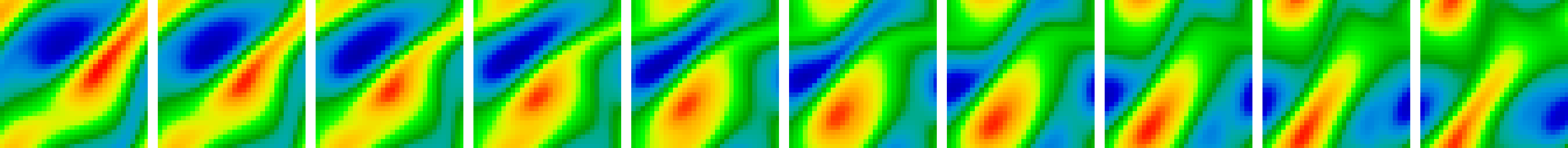} \\
        MAML & \includegraphics[align=c,width=0.8\linewidth]{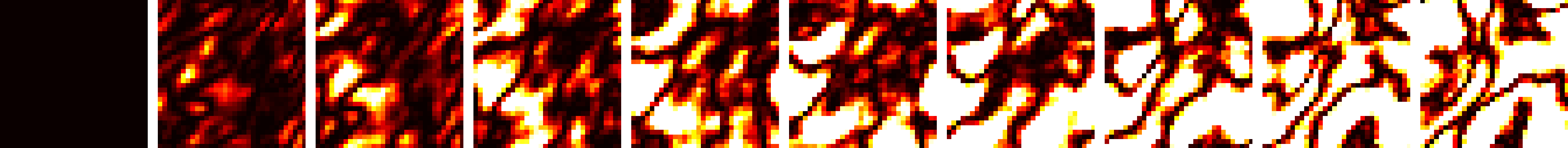} \\
        LEADS & \vspace{-0.5em}\includegraphics[align=c,width=0.8\linewidth]{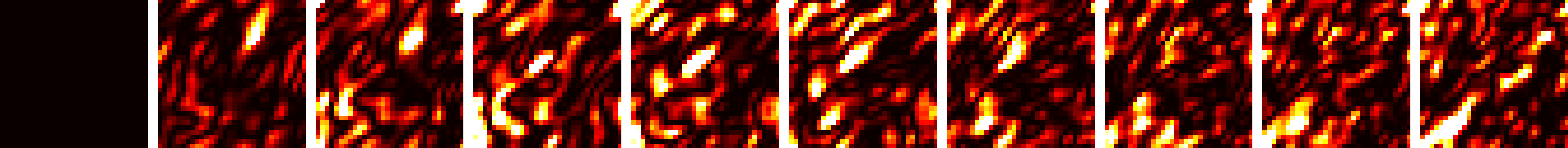} \\
        CAVIA-Concat & \vspace{-0.5em}\includegraphics[align=c,width=0.8\linewidth]{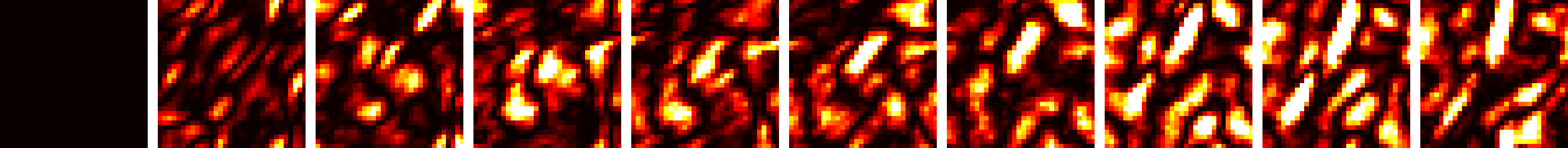}\\
        CoDA (\textit{Ours}) & \includegraphics[align=c,width=0.8\linewidth]{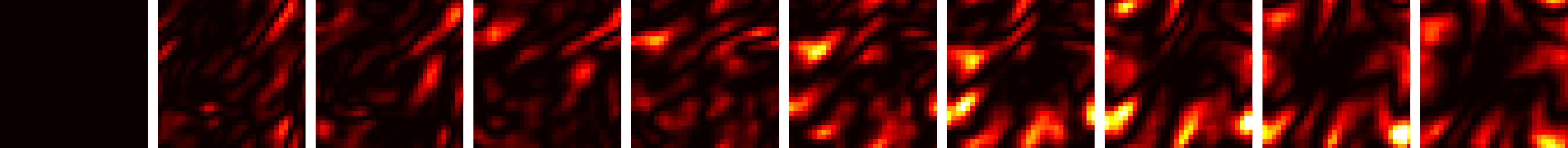}
    \end{tabular}
    \caption{Adaptation to new \texttt{NS} system - $\nu=1.15\cdot 10^{-3}$. Ground-truth trajectory and prediction MSE per frame for MAML, LEADS, CAVIA-Concat and CoDA.}
    \label{fig:ns_traj}
\end{figure*}

\end{document}